\documentclass{article}

\usepackage[a4paper, total={6in, 10in}]{geometry}
\usepackage{microtype}
\usepackage{graphicx}
\usepackage{subfigure}
\usepackage{booktabs} 
\usepackage{hyperref}

\usepackage{amsmath}
\usepackage{amssymb}

\usepackage{thmtools}
\usepackage{thm-restate}
\usepackage{algorithm}
\usepackage{algorithmic}
\usepackage{mathtools}
\usepackage{amsthm}
\usepackage{authblk}

\usepackage{enumitem}
\usepackage[capitalize,noabbrev]{cleveref}
\usepackage[textsize=tiny]{todonotes}
\usepackage[round]{natbib}

\newcommand*\samethanks[1][\value{footnote}]{\footnotemark[#1]}

\theoremstyle{plain}
\newtheorem{theorem}{Theorem}[section]
\newtheorem{proposition}[theorem]{Proposition}
\newtheorem{lemma}[theorem]{Lemma}

\theoremstyle{definition}

\newtheorem{assumption}{Assumption}
\theoremstyle{remark}

\def \R {\mathbb{R}}

\def \x {\mathbf{x}}

\def \E {\mathbb{E}}

\def \x {\mathbf{x}}

\def \xih {\widehat{\xi}}

\title{Online Allocation Problem with Two-sided Resource Constraints}

\author[1,2]{Qixin Zhang \thanks{Equal contribution.}}
\author[1]{Wenbing Ye \samethanks}
\author[1]{Zaiyi Chen \samethanks}
\author[1]{Haoyuan Hu}
\author[3]{Enhong Chen}
\author[2]{Yu Yang}

\affil[1]{ Cainiao Network, Hang Zhou, China}
\affil[2]{School of Data Science City University of Hong Kong, Kowloon Hong Kong, China}
\affil[3]{University of Science and Technology of China}
\begin{document}

\maketitle

\begin{abstract}
In this paper, we investigate the online allocation problem of maximizing the overall revenue subject to both lower and upper bound constraints. Compared to the extensively studied online problems with only resource upper bounds, the two-sided constraints affect the prospects of resource consumption more severely. As a result, only limited violations of constraints or pessimistic competitive bounds could be guaranteed. To tackle the challenge, we define a measure of feasibility $\xi^*$ to evaluate the hardness of this problem, and estimate this measurement by an optimization routine with theoretical guarantees. We propose an online algorithm adopting a constructive framework, where we initialize a threshold price vector using the estimation, then dynamically update the price vector and use it for decision-making at each step. It can be shown that the proposed algorithm is $\big(1-O(\frac{\varepsilon}{\xi^*-\varepsilon})\big)$ or $\big(1-O(\frac{\varepsilon}{\xi^*-\sqrt{\varepsilon}})\big)$ competitive with high probability for $\xi^*$ known or unknown respectively. To the best of our knowledge, this is the first result establishing a nearly optimal competitive algorithm for solving two-sided constrained online allocation problems with a high probability of feasibility.

\end{abstract}

\section{Introduction}
\label{sec: intro}
Online resource allocation is a prominent paradigm for sequential decision making during a finite horizon 
subject to the resource constraints, increasingly attracting the wide attention of researchers and practitioners in theoretical computer science \citep{mehta2007adwords,devanur2012online,devanur2019near}, operations research \citep{agrawal2014dynamic,li2021online} and machine learning communities \citep{balseiro2020dual,li2020simple}. In these settings, the requests arrive online and 
we need to serve each request via one of the available channels, which consumes a certain amount of resources and generates a corresponding service charge. The objective of the decision maker is to maximize the cumulative revenue subject to the resource capacity constraints. Such problem frequently appears in many applications including online advertising \citep{mehta2007adwords,buchbinder2007online}, online combinatorial auctions \citep{chawla2010multi}, online linear programming \citep{agrawal2014dynamic,buchbinder2009online}, online routing \citep{buchbinder2006improved}, online multi-leg flight seats and hotel rooms allocation \citep{talluri2004theory}, etc.

The aforementioned online resource allocation framework only considers the capacity (upper bound) constraints for resources. As a measure of fairness for resources expenditure, the requirements for guaranteeing a certain amount of resource allocation 
play important roles in real-world applications \citep{haitao2007fairness,zhang2020request} 
ranging from contractual obligations and group-level fairness to load balance. We give several examples in \cref{appendix:a} for completeness. Recently, some attempts come out to alleviate the difficulties introduced by lower bound constraints. For instance, \citet{lobos2021joint} propose an online mirror descent method to address this new online allocation problem with $O(\sqrt{T})$ asymptotic regret as well as $O(\sqrt{T})$ violation of lower bounds in expectation, where $T$ is the number of requests. Meanwhile, \citet{balseiro2021regularized} consider a more general regularized setting to satisfy fairness requirements by non-separable penalties, which also leads to $O(\sqrt{T})$ asymptotic regret by wisely choosing 
the penalties. All these studies consider the lower bound requirements as \textsl{soft} threshold, i.e., there is no guarantee on the satisfaction of lower bound constraints, which remains an open problem for online resource allocation.



In this paper, we investigate the online resource allocation problem under stochastic setting, where requests are drawn 
independently from some \textsl{unknown} distribution, and arrive sequentially. The goal is to maximize the total revenue subject to the two-sided resource constraints. In the beginning, the decision maker is endowed with a limited and unreplenishable amount of resources, and agrees a minimum consumption on each resource. The decision maker can access the information of the current request and requests that have been processed before, but she is unable to get the information of future requests until their arrivals. Once observing an online request, the decision is made irrevocably to assign one available channel to serve the request, where we assume there is an oracle that determines the revenue and the amount of consumed resources depending on (request, channel) pair for convenience. A formal definition can be found in \cref{sec:preliminaries}. 

To address the challenges brought by two-sided resource constraints, we first define a problem dependent quantity $\gamma$ to represent i). the maximum fraction of revenue or resource consumption per request, ii). the pessimism about the satisfaction of lower bound constraints. 
Generally, large $\gamma$ will rule out good competitive ratio across different settings \citep{mehta2007adwords,buchbinder2007online,agrawal2014dynamic,devanur2019near}. Meanwhile, we find that the margins between lower and upper bounds directly affect the hardness of the online allocation problem in stochastic setting. Inspired by the Slater's condition and its applications \citep{slater2014lagrange,boyd2004convex}, we define a measure of feasibility $\xi^*$ to evaluate the hardness and present an estimator with sufficient accuracy. In the end, the proposed algorithm integrates several estimators to compensate undiscovered information.
Our contribution can be summarized as follows:
\begin{enumerate}
    \item Under the gradually improved assumptions, we propose three algorithms that return $1-O(\frac{\varepsilon}{\xi^*-\varepsilon})$ competitive solutions \textsl{satisfying} the two-sided constraints w.h.p., if $\gamma$ is at most $O\left(\frac{\varepsilon^2}{\ln(K/\varepsilon)}\right)$, where $\xi^*\gg \varepsilon$ is the measure of feasibility and $K$ is the number of resources. To the best of our knowledge, this is the first result establishing a competitive algorithm for solving two-sided constrained online allocation problems \textsl{feasibly}, which is nearly optimal according to  \citep{devanur2019near}. 
    \item To tackle the unknown parameter $\xi^*$, we propose an optimization routine in \cref{alg:M}. Through merging the estimate method into the previous framework, a new algorithm is proposed in \cref{alg:5} with a solution achieving $1-O(\frac{\varepsilon}{\xi^* - \sqrt{\varepsilon}})$ competitive ratio when $\xi^*$ is an unknown and problem dependent constant.
    \item Our analytical tools can be used to strengthen the existing models \citep{mehta2007adwords,devanur2019near,lobos2021joint} for online resource allocation problems. 
\end{enumerate}

\section{Related Works}
\label[section]{sec:related work}
Online allocation problems have been extensively studied in theoretical computer science and operations research communities. In this section, we overview the related literature.

When the incoming requests are adversarially chosen, there is a stream of literature investigating online allocation problems. \citet{mehta2007adwords} and \citet{buchbinder2007online} first study the AdWords problem, a special case of online allocation, and provide an algorithm that obtains a $(1-1/e)$ approximation to the offline optimal allocation, which is optimal under the adversarial input model. However, the adversarial assumption may be too pessimistic about the requests. 
To consider another application scenarios, \citet{devanur2009adwords} propose the random permutation model, where an adversary first selects a sequence of requests which are then presented to the decision maker in random order. This model is more general than the stochastic i.i.d. setting in which requests are drawn independently and at random from an unknown distribution. In this new stochastic model, \citet{devanur2009adwords} revisit the AdWords problem and present a dual training algorithm with two phases: a training phase where data is used to estimate the dual variables by solving a linear program and an exploitation phase where actions are taken using the estimated dual variables. Their algorithm is guaranteed to obtain a $1-o(1)$ competitive ratio, which is problem dependent. \citet{36635} show that this training-based algorithm could resolve more general linear online allocation problems. Pushing these ideas one step further, \citet{agrawal2014dynamic} consider primal and dual  algorithm that dynamically updates dual variables by periodically solving a linear program using the data collected so far. 
Meanwhile, \citet{kesselheim2014primal} take the same policy and only consider renewing the primal variables. Recently, \citet{devanur2019near} take other innovative techniques to geometrically update the price vector via some decreasing potential function derived from probability inequalities. These algorithms also obtain $1-o(1)$ approximation guarantees under some mild assumptions. While the algorithms described above usually require solving large linear problems periodically, there is a recent line of work seeking simple algorithms that does not need to solve a large linear programming. \citet{balseiro2020dual} study a simple dual mirror descent algorithm for online allocation problems with concave reward functions and stochastic inputs, which attains $O(\sqrt{KT})$ regret, where $K$ and $T$ are the number of resources and requests respectively, i.e., updating dual variables via mirror descent algorithm
and avoids solving large auxiliary linear programming. Simultaneously, \citet{li2020simple} present a similar fast algorithm that updates the dual variable via projected gradient descent in every round  for linear rewards. It is worth noting that all of these literature only consider the online allocation with capacity constraints. 

\section{Preliminaries and Assumptions}\label{sec:preliminaries}
In this section, we introduce some essential concepts, assumptions, and concentration inequalities that are adopted in our framework.

\subsection{Two-sided Resource Allocation Framework}\label{subsec:framework}

We consider the following framework for offline resource allocation problems. Let $\mathcal{K}$ be the set of $K$ resources. There are $J$ different types of requests. Each request $j\in\mathcal{J}$ ($|\mathcal{J}|=J$) could be served via some channel $i\in\mathcal{I}$, which will consume $a_{ijk}$ amount of resource $k\in\mathcal{K}$ and generate $w_{ij}$ amount of revenue. 
For each resource $k\in\mathcal{K}$, $L_{k}$ and $U_{k}$ denote the lower bound requirement and capacity, respectively. The objective of the two-sided resource allocation is to maximize the revenue subject to the two-sided resource constraints. The following is the offline integer linear programming for resource allocation where the entire sequence of $T$ requests is given in advance:
\begin{equation}\label{lp:1} 
    \begin{aligned}
        W_{R}=&\max_{x}\ \sum_{i\in\mathcal{I},j\in[T]} w_{ij}x_{ij}\\
        s.t.\ &L_{k}\le\sum_{i\in\mathcal{I},j\in[T]} a_{ijk}x_{ij}\le U_{k}, \forall k\in\mathcal{K}\\
        &\sum_{i\in\mathcal{I}} x_{ij}\le 1, \forall j\in[T]\\
        &x_{ij}\in\{0,1\}, \forall i\in\mathcal{I}, j\in[T]
    \end{aligned}
\end{equation}
where we denote the sample set by $[T]$. It is important to distinguish the collection of all kinds of requests $\mathcal{J}$ and the sample set $[T]$. When $\mathcal{J}$ is used in the LP problem,  the (expected) number of requests of each type is required. As a common relaxation in online resource allocation literature \citep{agrawal2014dynamic,devanur2019near,li2020simple,balseiro2020dual}, not picking any channel is permitted in  ILP~\eqref{lp:1}. We denote the no picking channel option by $\perp\in\mathcal{I}$, where $a_{\perp jk}=0$ and $w_{\perp j}=0$ for all $j\in\mathcal{J}$ and $k\in\mathcal{K}$.

\subsubsection{Online Two-sided Resource Allocation}
To facilitate the analysis in following sections, we hereby explain the detailed assumptions for two-sided online resource allocation problem.

For better exhibition, we denote the following expected LP problem with lower bound $L_{k}+\beta T\bar{a}_{k}$ for every resource $k\in\mathcal{K}$ by $E(\beta)$ , i.e., 
   \begin{equation}\label{equ:E_beta}
    \begin{aligned}
      W_{\beta}=&\max_x\ \sum_{i\in\mathcal{I},j\in\mathcal{J}} Tp_{j}w_{ij}x_{ij}\\
        s.t.\ &L_{k}+\beta T\bar{a}_{k}\le\sum_{i\in\mathcal{I},j\in\mathcal{J}} Tp_{j}a_{ijk}x_{ij}\le U_{k}, \forall k\in\mathcal{K}\\
        &\sum_{i\in\mathcal{I}} x_{ij}\le 1, \forall j\in\mathcal{J}\\
        &x_{ij}\ge 0, \forall i\in\mathcal{I}, j\in\mathcal{J}    
    \end{aligned}
  \end{equation}where $\beta\in [0, \xi^*]$ is a deviation parameter and $W_{\beta}$ is the optimal value for problem $E(\beta)$. Meanwhile, we denote the optimal solution for $E(\beta)$ as $\{x(\beta)_{ij}^{*},\forall i\in\mathcal{I}, j\in\mathcal{J}\}$. The motivation for adding $\beta T \bar{\alpha}_k$ to the lower bound is to \emph{measure} how the gap between the lower and upper bounds affects the difficulty of the problem. More insights can be found in the proof of \cref{thm:1}, where we solve a factor-revealing LP problem and then build the connection between $\xi^*$ (in \cref{assumption:3}) and the competitive ratio.

\begin{assumption}\label{assumption:1}
The requests arrive sequentially and are independently drawn from some unknown distribution $\mathcal{P}:\mathcal{J}\rightarrow [0,1]$, where $\mathcal{P}(j)$ denotes the arrival probability of request $j\in\mathcal{J}$ and we denote $p_{j}=\mathcal{P}(j)$, $\forall j\in\mathcal{J}$.    
\end{assumption}

Under \cref{assumption:1}, if we take the same policy for every request $j\in \mathcal{J}$ in ILP~\eqref{lp:1}, it can be verified that $\mathbb{E}\left[\sum_{i\in\mathcal{I},j\in[T]} a_{ijk}x_{ij}\right]=\sum_{i\in\mathcal{I},j\in\mathcal{J}} Tp_{j}a_{ijk}x_{ij}$ and $\mathbb{E}\left[\sum_{i\in\mathcal{I},j\in[T]} w_{ij}x_{ij}\right]=\sum_{i\in\mathcal{I},j\in\mathcal{J}} Tp_{j}w_{ij}x_{ij}$. Therefore, we could view the LP~\eqref{equ:E_beta} with $\beta=0$ as a relaxed version of the expectation of ILP~\eqref{lp:1}. Moreover, $W_0$ is an upper bound of the expectation of $W_{R}$.

\begin{restatable}{lemma}{objlemma}\label{lemma:1} $W_0\ge\mathbb{E}\left[W_{R}\right]$.
\end{restatable}
The proof of the above lemma is deferred to \cref{appendix:lem1}. The competitive ratio of an algorithm is defined as the ratio of the expected cumulative revenue of the algorithm to $W_0$.

\begin{assumption}\label{assumption:2} 
In stochastic settings, we make the following reasonable assumptions.
\begin{enumerate}  
    \item[i.] When the algorithm is initialized, we know the lower and upper bound requirements $L_{k}$ and $U_{k}$ regarding every resource $k\in\mathcal{K}$ and the number of requests $T$.
    \item[ii.] Without loss of generality, the revenues $w_{ij}$ and consumption of resources $a_{ijk}$ are finite, non-negative and revealed when each request arrives, $\forall i\in\mathcal{I}, j\in\mathcal{J}, k\in\mathcal{K}$. Moreover, we know $\bar{w}= \sup_{i\in\mathcal{I},j\in\mathcal{J}} w_{ij}$ and $\bar{a}_{k}=\sup_{i\in\mathcal{I},j\in\mathcal{J}} a_{ijk},\forall k\in\mathcal{K}$.
\end{enumerate}
\end{assumption} 
The above assumptions are widely adopted in literature. We also make an assumption on the margin of every resource constraint.
\begin{assumption}\label{assumption:3}(\textrm{Strong feasible condition}) 

There exists a $\xi>0$ making the linear constraints of the following problem feasible,
   \begin{equation}\label{equ:5}
    \begin{aligned}
        &\xi^* = \max_{\xi\geq 0,x} \xi\\
        s.t.\ &L_{k}+\xi  T\bar{a}_{k}\le\sum_{i\in\mathcal{I},j\in\mathcal{J}} Tp_{j}a_{ijk}x_{ij}\le U_{k}, \forall k\in\mathcal{K}\\
        &\sum_{i\in\mathcal{I}} x_{ij}\le 1, \forall j\in\mathcal{J}\\
        &x_{ij}\ge 0, \forall i\in\mathcal{I}, j\in\mathcal{J}.
    \end{aligned}
\end{equation}
We call $\xi^*$ the {\bf measure of feasibility}, and 
assume $\xi^*\gg \varepsilon$ for simplicity, where $\varepsilon>0$ is an error parameter. 
\end{assumption}

Due to limited space, we present some frequently used concentration inequalities to predigest the theoretical analysis in \cref{appendix:concentration}.

\section{Competitive Algorithms for Online Resource Allocation with Two-sided Constraints}\label{sec:competitive alg}
In this section, we propose a series of online algorithms for resource allocation with two-sided constraints by progressively weakening the following assumptions.

\ref{sec:known distribution} and \ref{sec:sen}. With known distribution;

\ref{sec:optimal}. With known optimal objective;

\ref{sec:unknow}. Completely unknown distribution.


\subsection{With Known Distribution}\label{sec:known distribution}
In this section, we assume that we have the complete knowledge of the distribution $\mathcal{P}$. We first propose a high-level overview of our algorithm and outcomes as follows.

\emph{\textbf{High-level Overview:}}
\begin{enumerate}
    \item With the knowledge of $\mathcal{P}$, we could directly solve the expected problem $E(\tau)$ and obtain the optimal solution $\{x(\tau)_{ij}^{*},\forall i\in\mathcal{I}, j\in\mathcal{J}\}$, where the deviation parameter $\tau=\frac{\varepsilon}{1-\varepsilon}$.  For each fixed request $j\in\mathcal{J}$, if $x(\tau)_{i_{1}j}^{*}\ge x(\tau)_{i_{2}j}^{*}$, we tend to assign this type of requests to channel $i_1$ rather than $i_{2}$. Motivated by this intuition, we design the Algorithm $\widetilde{P}$ in \cref{alg:1}, which assigns request $j\in\mathcal{J}$ to channel $i\in\mathcal{I}$ with probability $(1-\varepsilon)x(\tau)_{ij}^*$.
    \item According to the Algorithm $\widetilde{P}$, if we define the r.v. $X_{\cdot k}^{\widetilde{P}}=\sum_{i\in\mathcal{I}}a_{i\cdot k}x_{i\cdot k}$ for resource $k$ consumed by one request sampled from distribution $\mathcal{P}$ and r.v. $Y^{\widetilde{P}}$ for the revenue, it is easy to obtain that $(1-\varepsilon)\frac{L_{k}}{T}+\varepsilon \bar{a}_{k}\le\mathbb{E}(X_{k}^{\widetilde{P}})\le(1-\varepsilon)\frac{U_{k}}{T}$ and $\E[Y^{\widetilde{P}}]=(1-\varepsilon)\frac{W_{\tau}}{T}$. Because the expectation of resource consumption is restricted to the interval $[\frac{L_{k}}{T}, \frac{U_{k}}{T}]$, in \cref{thm:1} we could prove that the Algorithm $\widetilde{P}$ generates a feasible solution with high probability for online resource problem with two-sided constraints ILP~\eqref{lp:1} via the Bernstein inequalities in Lemma~\ref{lemma:2}. This is the reason for enlarging the lower bound $L_{k}$ by $\tau T\bar{a}_{k}$ and scaling the solution with a factor $1-\varepsilon$. Meanwhile, we also verify that the accumulative revenue will be no less than $(1-2\varepsilon)W_{\tau}$ w.p. $1-\varepsilon$ in Lemma~\ref{thm:2}.
    \item  Since we have lifted the lower resource constraints from $L_{k}$ to $L_{k}+\tau T\bar{a}_{k}$ in Algorithm $\widetilde{P}$, it would cause the change of baseline when analyzing the accumulative revenues. We have to derive the relationship between $W_{\tau}$ and $W_{0}$ ($W_0$). By taking a sensitive analysis in \cref{sec:sen}, we obtain $W_{\tau}\ge\big(1-\frac{\tau}{\xi^*}\big)W_0$ in \cref{lemma:3} under the \cref{assumption:3}, where $\xi^*$ is the measure of feasibility. Finally, we could prove Theorem~\ref{thm:1}. 
\end{enumerate}

\begin{algorithm}[t]\small
	\caption{Algorithm $\widetilde{P}$ }\label{alg:framework}
	\begin{algorithmic}[1]\label{alg:1}
	\STATE {\bf Input:} $\tau=\frac{\varepsilon}{1-\varepsilon}$, $\mathcal{P}$
    \STATE {\bf Output:} $\{x_{ij}\}_{i\in\mathcal{I},j\in[T]}$
	    \STATE $\{x(\tau)_{ij}^*\}_{i\in\mathcal{I},j\in\mathcal{J}}= \arg\max_{\{\x\}} E(\tau)$.
		\STATE When a type $j\in\mathcal{J}$ request comes, we assign this request to channel $i$ with probability $(1-\varepsilon)x(\tau)_{ij}^*$. That is, If assigning the type $j$ request to channel $i$, we set $x_{ij}=1$, otherwise $x_{ij}=0$.
 	\end{algorithmic}
\end{algorithm}

Similar to \citep{devanur2019near}, we first consider the competitive ratio that \cref{alg:1} could achieve for a surrogate LP problem $E(\tau)$ with optimal objective $W_\tau$ according to Definition \eqref{equ:E_beta}. 
\begin{restatable}{lemma}{thmkowndis}\label{thm:2}
 Under \cref{assumption:1}-\ref{assumption:3}, if $\forall\varepsilon> 0$ and $\gamma=O(\frac{\varepsilon^2}{\ln(K/\varepsilon)})$, \cref{alg:1} achieves an objective value at least $(1-2\varepsilon)W_{\tau}$ and satisfies the constraints w.p. $1-\varepsilon$.
\end{restatable}
The proof is deferred to \cref{appendix:thm1}.
From \cref{thm:2}, we have verified the cumulative revenue is at least $(1-2\varepsilon)W_{\tau}$, w.p. $1-\varepsilon$. Thus, in order to compare the revenue with $W_0$, we should derive the relationship between $W_{\tau}$ and $W_0$. However, due to the deviation $\tau T\bar{a}_{k}$, it is hard to directly obtain this relationship. We will tackle this sensitive problem in the next subsection.

\subsection{Sensitive Analysis}
\label{sec:sen}
In this subsection, we demonstrate the relationship between $W_{\tau}$ and $W_0$. Before introducing the details, we first investigate the difference between the problem $E(\tau)$ and the problem $E(0)$. Specifically, we enlarge the lower bound for every resource $k\in\mathcal{K}$ by an extra amount of $\tau T\bar{a}_{k}$, and then investigate the effects of these changes.

Through a sensitive analysis, we finally found that $\xi^*$ controls the decline ratio of the $E(\tau)$ objective function value. The result can be summarized by follows.
\begin{restatable}{theorem}{factorlemma}\label{lemma:3}
Under the strong feasible condition in \cref{assumption:3}, the optimal objective of $E(\tau)$ satisfies that
\begin{displaymath}
W_{\tau}\ge\bigg(1-\frac{\tau}{\xi^*}\bigg)W_0,
\end{displaymath}
where $\tau=\frac{\varepsilon}{1-\varepsilon}$.
\end{restatable}
We prove \cref{lemma:3} from a geometrical perspective.
\begin{proof}
Let  
$x_{ij}^{'} =  (1 - \frac{\tau}{\xi^{*}}) x(0)_{ij}^{*} + \frac{\tau}{\xi^{*}}x(\xi^{*})_{ij}^{*}$, $\forall i\in \mathcal{I}$ and $j\in\mathcal{J}$, then $x_{ij}^{'}$
is a convex combination between the optimal solutions of $E(0)$ and $E(\xi^{*})$. 
 According to the constraint set of problem $E(0)$ and $E(\xi)$, we can verify that $x_{ij}^{'}$ is non-negative and satisfies 
    $\sum_{i\in\mathcal{I},j\in\mathcal{J}} Tp_{j}a_{ijk}x_{ij}\le U_{k}, \forall k\in\mathcal{K}$,
        $\sum_{i\in\mathcal{I}} x_{ij}\le 1, \forall j\in\mathcal{J}$.

Meanwhile, 
\begin{equation}\label{equ:2combination}
    \begin{aligned}
        &\sum_{i\in\mathcal{I},j\in\mathcal{J}} Tp_{j}a_{ijk}x_{ij}^{'}\\
        &\ge \sum_{i\in\mathcal{I},j\in\mathcal{J}} Tp_{j}a_{ijk}\left((1 - \frac{\tau}{\xi^{*}}) x(0)_{ij}^{*} + \frac{\tau}{\xi^{*}}x(\xi^{*})_{ij}^{*}\right)\\
        &\ge (1 - \frac{\tau}{\xi^{*}})L_k + \frac{\tau}{\xi^{*}}(L_k + \xi^{*} T \bar{a}_k)\\
        &\ge L_k + \tau T \bar{a}_k, \forall k \in \mathcal{K}.
    \end{aligned}
\end{equation}
Thus $x_{ij}^{'}$ is feasible to $E(\tau)$, and 
$$W_{\tau} \ge \sum_{i\in\mathcal{I},j\in\mathcal{J}} Tp_{j}w_{ij}x_{ij}^{'} \ge (1- \frac{\tau}{\xi^{*}})W_0. $$
This completes the proof of \cref{lemma:3}.
\end{proof}

In practical problems, the influence of $\xi^*$ on competitive ratio may be far better than the worst case bound in \cref{lemma:3}, since constraints usually represent different resource requirements and only affect a part of requests. It is worth mentioning that our initial proof of \cref{lemma:3} was based on analyzing a factor-revealing fractional linear programming problem motivated by \citep{jain2003greedy}. Although the above geometric proof of \cref{lemma:3} is simpler, the factor-revealing perspective tells that $\xi^*$ represents the severity of mutual interference between the lower and upper bounds. Its effect will eventually be reflected in the competitive ratio. We put this analysis in \cref{appendix::factor_proof} due to the space limit and wish to provide more insights.


Combining \cref{thm:2} with \cref{lemma:3}, we can show that the cumulative revenue obtained by Algorithm $\widetilde{P}$ is larger than $(1-2\varepsilon)W_{\tau}\ge(1-2\varepsilon)(1-\frac{\tau}{\xi^*})W_0\ge\big(1-(2+\frac{1}{\xi^*})\varepsilon\big)W_0$. Therefore, we have the following theorem.
\begin{restatable}{theorem}{thmknowndis}\label{thm:1} 
Under \cref{assumption:1}-\ref{assumption:3}, if $\varepsilon> 0$, $\tau=\frac{\varepsilon}{1-\varepsilon}$ and  $\gamma=\max\big(\frac{\bar{a}_{k}}{U_{k}}, \frac{\bar{a}_{k}}{T\bar{a}_{k}-L_{k}}, \frac{\bar{w}}{W_{\tau}} \big)=O\big(\frac{\varepsilon^2}{\ln(K/\varepsilon)}\big)$, \cref{alg:framework} achieves an objective value of at least $\big(1-(2+\frac{1}{\xi^*})\varepsilon\big)W_0$ and satisfies the constraints w.p. $1-\varepsilon$.
\end{restatable}
Although $\widetilde{P}$ is an impractical algorithm owing to the complete knowledge of distribution $\mathcal{P}$, it builds a bridge to the desired competitive ratio. In the following subsections, we will release the unavailable knowledge by replacing $\widetilde{P}$ in a constructive way.

\subsection{With Known $W_{\tau}$} 
\label{sec:optimal}
In this section, we only assume the knowledge of the optimal value $W_{\tau}$ and abandon the assumption of knowing the distribution $\mathcal{P}$. Based on this assumption, we design an algorithm $A$ only using $W_{\tau}$ in \cref{alg:2}, which achieves at least $\big(1-(2+\frac{1}{\xi^*})\varepsilon\big)W_0$ objective and satisfies the constraints w.p. $1-\varepsilon$. 

Intuitively, we hope to construct a new algorithm $A$ through Algorithm $\widetilde{P}$, which performs no worse than $\widetilde{P}$ with high probability. We still use the r.v. $X^{A}_{jk}$ for resource $k$ consumed by the $j$-th request, which is determined by algorithm~$A$, and r.v. $Y^{A}_j$ for the revenue. We first investigate a simple case with only one upper bound $U_k$. Considering the intermediate step in the proof of Bernstein inequality, we can bound the probability of failing to follow this upper bound constraint by
\begin{align*}
    &P\bigg(\sum_{j=1}^TX_{jk}^{\widetilde{P}}\ge U_k\bigg)\\
    &\leq\min_{t>0}\E\Bigg[\exp\bigg(t(\sum_{j=1}^{s}X^{\widetilde{P}}_{jk}-\frac{s}{T}U_{k})+t(\sum_{j=s+1}^{T}X^{\widetilde{P}}_{jk}\\
    &-\frac{T-s}{T}U_k)\bigg)\Bigg].
\end{align*}
Then we relax the second term $t(\sum_{j=s+1}^{T}X^{\widetilde{P}}_{jk}-\frac{T-s}{T}U_k)$ by \cref{lemma:2}.i. with $t$ given in the proof of \cref{thm:2}, which means we bound the failure probability for requests that still decided by Algorithm $\widetilde{P}$ from $s+1$ to $T$. Although $\widetilde{P}$ is unknown, its failure probability has been bounded explicitly. We construct Algorithm $A$ by induction. In the beginning, 
Algorithm $A$ obtains the information of the first request, and determines $X_{1k}^A$ no worse than the decision made by Algorithm $\widetilde{P}$. To achieve this goal, Algorithm $A$ iterates over all possible channels to minimize the failure probability upper bound (similar to \cref{alg:2}), replacing $X_{1k}^{\widetilde{P}}$ by $X_{1k}^A$ in the above expression. In the induction step, $s>1$, since Algorithm $A$ has decided the first $s-1$ requests, it will make a decision for the $s$-th request no worse than Algorithm $\widetilde{P}$. 


In general, the key to the previous analysis is bounding the probabilities of three \textbf{bad} events: i). $\sum_{j=1}^{T} X^{A}_{jk}\ge U_{k}, \forall k\in\mathcal{K}$; ii). $\sum_{j=1}^{T} X^{A}_{jk}\le L_{k}, \forall k\in\mathcal{K}$; iii). $\sum_{j=1}^{T} Y^{A}_{j}\le (1-2\varepsilon)W_{\tau}$. The joint probability of the events can be considered as the \textsl{failure} probability of the proposed algorithm. We use the Bernstein inequalities to bound the failure probability of the Algorithm $\widetilde{P}$. 
We first design a moment generating function $\mathcal{F}(A^s\widetilde{P}^{T-s})$ for controlling the probability of the three bad events of the hybrid Algorithm $A^{s}\widetilde{P}^{T-s}$, which runs Algorithm $A$ for the first $s$ requests and $\widetilde{P}$ for the rest $T-s$ requests. We let $\mathcal{F}$ enjoy an inspiring 
monotone property for the special hybrid Algorithm $A^{s}\widetilde{P}^{T-s}$, i.e., $\mathcal{F}(A^{s+1}\widetilde{P}^{T-s-1})\le\mathcal{F}(A^{s}\widetilde{P}^{T-s})$ and $\mathcal{F}(\widetilde{P}^{T})\le\varepsilon$. Hence, the $\mathcal{F}(A^{T})\le\mathcal{F}(A^{T-1}\widetilde{P})\le\dots\le\mathcal{F}(\widetilde{P}^{T})\le\varepsilon$. Intuitively, Algorithm $A$ can be viewed as minimizing the upper bound of the probability of bad events. This completes the high-level idea when $W_\tau$ is known.
We present the details of Algorithm $A$ in \cref{alg:2}.


\begin{algorithm}[t]
	\caption{Algorithm $A$}
	\begin{algorithmic}[1]\label{alg:2}
	\STATE {\bf Input:} $\varepsilon$, $W_{\tau}$\\
    \STATE {\bf Output:} $\{x_{ij}\}_{i\in\mathcal{K},j\in[T]}$
    \STATE Set $c_{1k}=\frac{-\ln(1-\varepsilon)}{\bar{a}_k}, \forall k \in \mathcal{K}$ and $c_2 = \frac{-\ln(1-\varepsilon)}{\bar{w}}$
	 \STATE Initiate $\phi_k^0 = 1, \varphi_k^0 = 1, \forall k \in \mathcal{K}$, and $\psi^0 = 1$
	 \FOR{$j = 1,\ldots, T$}
	 \STATE compute the optimal $i^*$ by
	 \begin{equation*}
    \begin{aligned}
    i^{*}=\arg\min_{i\in\mathcal{I}}&\sum_{k\in\mathcal{K}}\phi_k^{j-1}\exp\bigg(c_{1k}\big(a_{ijk} - \frac{U_{k}}{T}\big)\bigg)\\
        &+ \sum_{k\in\mathcal{K}}\varphi_k^{j-1}\exp\bigg(c_{1k}\big(\frac{L_k}{T} - a_{ijk}\big)\bigg) \\
        &+ \psi^{j-1}\exp\left(c_2\big(\frac{(1 - 2\varepsilon)W_{\tau}}{T} - w_{ij}\big)\right)
    \end{aligned}
\end{equation*}
	 \STATE Set  $X^{A}_{jk}=a_{i^{*}jk}$,$Y^{A}_{j}=w_{i^{*}j}$
	 \STATE Update $\phi_{k}^{j}=\phi_{k}^{j-1}\exp\big(c_{1k}(X^{A}_{jk}-\frac{U_k}{T})\big),\forall k \in \mathcal{K}$
	 \STATE Update $\varphi_{k}^{j}=\varphi_{k}^{j-1}\exp\big(c_{1k}(\frac{L_k}{T}-X^{A}_{jk})\big),\forall k \in \mathcal{K}$
	 \STATE Update $\psi^{j}=\psi^{j-1}\exp\big(c_2(\frac{(1-2\varepsilon)W_{\tau}}{T}-Y^{A}_{j})\big)$
	\ENDFOR
	\end{algorithmic}
\end{algorithm}

\begin{restatable}{theorem}{thmknownobj}\label{thm:3} 
Under \cref{assumption:1}-\ref{assumption:3}, if $\varepsilon> 0$, $\tau$ and $\gamma$ are defined as \cref{thm:1}, the \cref{alg:2}  achieves an objective value  at least $\big(1-(2+\frac{1}{\xi^*})\varepsilon\big)W_0$ and satisfies the constraints w.p. $1-\varepsilon$.
\end{restatable}
The proof is deferred to \cref{appendix:thm3}.

\subsection{With Known $\xi^*$ but Unknown Distribution}
\label{sec:unknow}

\begin{algorithm}[t]
\caption{Objective\_Estimator$(\{\mathrm{request}_j\}_{r}, t_{r},\beta)$}
\begin{algorithmic}[1]\label{alg: obj estimator}
    \STATE \textbf{Input:} requests from $r$-th stage $\{\mathrm{request}_j\}_{r}$, request number $t_{r}$, deviation parameter $\beta$.
    \STATE \textbf{Output:} $W^{r}$
    \STATE Solve $E(\beta)$ as 
    \begin{equation}\label{equ:23}
	\begin{aligned}
		 W^{r}=&\max_x\ \sum_{i\in\mathcal{I}}\sum_{j=1}^{t_{r}} w_{ij}x_{ij}\\
        s.t.\ &\frac{t_{r}}{T}(L_{k}+\beta T\bar{a}_{k})\le\sum_{i\in\mathcal{I}}\sum_{j=1}^{t_{r}} a_{ijk}x_{ij} \\
        &\le\frac{t_{r}}{T}U_{k},\forall k\in\mathcal{K}\\
        &\sum_{i\in\mathcal{I}} x_{ij}\le 1, \forall j\in[t_{r}]\\
        &x_{ij}\ge 0, \forall i\in\mathcal{I}, j\in[t_{r}]
	\end{aligned}
\end{equation} 
\end{algorithmic}
\end{algorithm}

In this section, we consider the completely unknown distribution setting. Under this setting, we first divide the incoming $T$ requests into multiple stages and run an inner loop similar to \cref{alg:2} in each stage. The differences between the inner loop and \cref{alg:2} are three-fold: i). choose the deviation parameter $\beta=\varepsilon$ instead of $\tau$, where $\varepsilon$ is an error parameter defined in \cref{thm:4}; ii). pessimistically reduce the estimated objective value from $W^{r-1}$ to $Z^r$ in each stage for boosting the chance of success, where $W^{r-1}$ is estimated by \cref{alg: obj estimator}; iii). consider the relative error of objective and the maximum relative error of constraints as $\varepsilon_{y,r}$ and $\varepsilon_{x,r}$ separately in each stage, instead of using the global error parameter $\varepsilon$. We consider that $\epsilon_{y,r}$ describes the error we care about the most (objective) and $\epsilon_{x,r}$ describes the rest errors (constraints and objective estimates). Distinguishing these error terms allows for more refined results, but 
will lead to the competitive ratio of the same order in our settings. In \cref{alg:3}, we name the proposed algorithm by $A_1$ to facilitate the analysis.

The high-level ideas can be summarized as follows. There are three types of errors that have to be restrained, 1). the error $|Z_r-W_\epsilon|$ from estimating deviant objective value, 2). the error $\epsilon_{x,r}$ from guaranteeing the satisfaction of constraints, and 3). the error $\epsilon_{y,r}$ from Algorithm $A_1$ affected by the randomness in the objective. We use an iterative learn-and-predict strategy to periodically reduce relative errors caused by insufficient samples. Let $\widehat{W}^r$ be the objective obtained by Algorithm $A_1$ in stage $r$. According to the previous sections, if objective $Z_r$ is reachable, we could prove that the loss $\ell_r \coloneqq|\widehat{W}^r - Z^r|$ is upper bounded by $O(\frac{t_r}{T}\epsilon_{y,r}Z_r)$ in each stage. We also need $Z_r$ to be a good estimate for the objective. $Z^r \geq (1-O(\epsilon_{x,r-1}))W_\epsilon$ ensures that the objective is not underestimated, and $Z_r\leq W_\epsilon$ ensures that it is a reachable target and facilitates the proof for $r=0,\ldots,l-1$. Besides the normal stages, a warm-up stage $r=-1$ is added to provide an estimate $Z_{-1}$. Then the cumulative loss $\ell_{-1} + \sum_{r=0}^{l-1}\ell_r$ will be no greater than $O\big(\frac{t_{-1}}{T}W_0 + \sum_r\frac{t_r(\epsilon_{x,r-1}+\epsilon_{y,r})}{T}W_\epsilon\big)$. The algorithm is designed to balance the relative error $\epsilon_{x,r}$, $\epsilon_{y,r}$ and their impact on the requests in each stage.

More precisely, Algorithm $A_{1}$ geometrically divides $T$ requests into $l+1$ stages, where the number of requests are $t_r=\varepsilon 2^{r} T$ for $r=0,\ldots, l-1$ and $t_{-1}=\varepsilon T$. In the initial stage $r=-1$, we use the first $t_{r}=\varepsilon T$ requests to estimate $W_\varepsilon$ and obtain $W^{-1}$, assuming that none of the requests are served in worst case. In stage $r\in\{0,1,\dots, l-1\}$, the requests from stage $r-1$ are used to provide more and more accurate estimate $W^{r-1}$ of $W_{\varepsilon}$. By Union Bound, we set the failure probability as $\delta=\frac{\varepsilon}{3l}$ and reduce the estimate $W^{r-1}$ to $Z^r=\frac{TW^{r-1}}{t_r(1+O(\varepsilon_{x,r-1}))}$, which is promised between $\left[(1-O(\varepsilon_{x,r-1}))W_{\varepsilon}, W_{\varepsilon}\right]$ w.p. $1-2\delta$.  

Then in each stage $r$, define the error parameters for objective and constraints as   $\varepsilon_{y,r}=O\left(\sqrt{\frac{T \ln(K/\delta)}{t_{r}Z_{r}}}\right)$ and $\varepsilon_{x,r}=O\left(\sqrt{\frac{\gamma_{1} T \ln(K/\delta)}{t_{r}}}\right)$ respectively. We construct a surrogate Algorithm $\widetilde{P}_2$ that achieves $\frac{t_{r}}{T}(1-\varepsilon_{y,r})Z^r$ cumulative revenue with the consumption of every resource $k$ between $\left[\frac{t_{r}(1+\varepsilon_{x,r})}{T}\big(L_{k}+(\varepsilon-\frac{\varepsilon_{x,r}}{1+\varepsilon_{x,r}})T\bar{a}_{k}\big), \frac{t_{r}(1+\varepsilon_{x,r})}{T}U_{k}\right]$ with probability at least $1-\delta$. Similar to previous sections, the connection between the inner loop of Algorithm $A_1$ and Algorithm $\widetilde{P}_2$ is built to minimize the upper bound of failure probabilities.
Finally, with probability at least $1-\delta$, the cumulative revenue is at least $\sum_{r=0}^{l-1}\frac{t_{r}Z_{r}}{T}(1-\varepsilon_{y,r})$ and the cumulative consumed resource of $k\in\mathcal{K}$ is between $\sum_{r=0}^{l-1} \frac{t_{r}(1+\varepsilon_{x,r})}{T}(L_{k}+(\varepsilon-\frac{\varepsilon_{x,r}}{1+\varepsilon_{x,r}})T\bar{a}_{k})$ and $\sum_{r=0}^{l-1}\frac{t_{r}(1+\varepsilon_{x,r})}{T}U_{k}$. Letting $\gamma_{1}=O(\frac{\varepsilon^{2}}{\ln(K/\varepsilon)})$, we could keep  $\sum_{r=0}^{l-1}\frac{t_{r}}{T}(1+\varepsilon_{x,r})U_{k}\le U_{k}$ ,$\sum_{r=0}^{l-1} \frac{t_{r}(1+\varepsilon_{x,r})}{T}(L_{k}+(\varepsilon-\frac{\varepsilon_{x,r}}{1+\varepsilon_{x,r}})T\bar{a}_{k})\ge L_{k}$ and $\sum_{r=0}^{l-1}\frac{t_{r}Z_{r}}{T}(1-\varepsilon_{y,r})\ge\big(1-O(\frac{\varepsilon}{\xi^*-\varepsilon})\big)W_0$. 

The most tricky part of Algorithm $A_1$ is estimating $W^r$ in \cref{alg: obj estimator}. Since the lower bound $L_k$ cannot upper bound the mean of $X_{jk}$ as $U_k$ does, it brings a great challenge to the theoretical analysis. To address this issue, we solve a biased LP problem, uplifting $L_k$ by $\epsilon T\bar{a}_k$, at the end of each stage. It helps the satisfaction of the lower bound in the next stage, and the influence on the objective can be determined by \cref{lemma:3}. This completes the high-level overview of the analysis of Algorithm $A_{1}$. The theoretical result of Algorithm $A_1$ is presented in \cref{thm:4} with explicit requirements of parameters.

\begin{algorithm}[t]\small
	\caption{Algorithm $A_{1}$}
    \begin{algorithmic}[1]\label{alg:3}
	\STATE{\bf Input:} $\varepsilon$, $\gamma_1$, $\xi^*$\\
    \STATE{\bf Output:} $\{x_{ij}\}_{i\in\mathcal{K},j\in[T]}$
	\STATE Set $l=\log_{2}(\frac{1}{\varepsilon})$, $t_{r}=\varepsilon2^{r}T$, $t_{-1}=\varepsilon T$ and $\delta=\frac{\varepsilon}{3l}$
	\FOR{$r=0$ to $l$-1}
	\STATE Set $W^{r-1}$ =Objective\_Estimator($\{\mathrm{request}_j\}_{r-1}, t_{r-1},\varepsilon$)
  \STATE Set $Z_{r}=\frac{T W^{r-1}}{(1+(2+\frac{1}{\xi^*-\varepsilon})\varepsilon_{x,r-1})t_{r-1}}$ 
  \STATE Set $\varepsilon_{x,r}=\sqrt{\frac{4T\gamma_1 \ln(\frac{2K + 1}{\delta})}{t_r}}$, $\varepsilon_{y,r}=\sqrt{\frac{4T\ln(\frac{2K+1}{\delta})\bar{w}}{Z_rt_r}}$
  \STATE Set $c_{1k,r} = \frac{\ln(1+\varepsilon_{x,r})}{\bar{a}_k}$ and $c_{2,r} = \frac{\ln(1+\varepsilon_{y,r})}{\bar{w}}$
	\STATE Initialize $\phi_{k}^0=\exp\Big(\frac{-(t_r-1)\varepsilon_{x,r}^2}{4\gamma_{1}T}\Big)$,\\$\varphi_{k}^0=\exp\Big(\frac{-(t_r-1)\varepsilon_{x,r}^2}{4\gamma_{1}T}\Big)$,$\psi^0=\exp\Big(\frac{-(t_r-1)\varepsilon^2_{y,r}Z^r}{4\bar{w}T}\Big)$
	 \FOR{$j=1,\ldots, t_r$}
	 \STATE compute the optimal $i^*$ by
	 \begin{align*}
	  i^*=&\arg\min_{i\in \mathcal{I}}\bigg\{\sum_{k\in\mathcal{K}}\phi_{k}^{j-1}\exp\Big(c_{1k,r}\big(-\frac{(1+\varepsilon_{x,r})U_k}{T}\\&+a_{ijk}\big)\Big)+\sum_{k\in\mathcal{K}}\varphi_{k}^{j-1}\exp\Big(c_{1k,r}\big(\bar{a}_{k}-a_{ijk}-\\&\frac{(1+\varepsilon_{x,r})((1-\varepsilon)T\bar{a}_{k}-L_{k})}{T}\big)\Big)+\\&\psi^{j-1}\exp\Big(c_{2,r}\big(\frac{(1-\varepsilon_{y,r})Z^r}{T} - w_{ij}\big)\Big)\bigg\}
	  \end{align*}
	 \STATE Set $X^{A_1}_{jk}=a_{i^{*}jk}$,$Y^{A_1}_{j}=w_{i^{*}j}$, $Z_{jk}^{A_1} = \bar{a}_k - a_{i^{*}jk}$
	 \STATE Update  
	 \STATE $\phi_{k}^{j}=\phi_{k}^{j-1}\exp\Big(c_{1k,r}(X^{A_1}_{tk}-\frac{(1+\varepsilon_{x,r})U_k}{T})+\frac{\varepsilon^2_{x,r}}{4T\gamma_1}\Big)$,
	 \STATE  $\varphi_{k}^{j}=\varphi_{k}^{j-1}\exp\Big(-c_{1k,r}(\frac{(1+\varepsilon_{x,r})((1-\varepsilon)T\bar{a}_{k}-L_{k})}{T} - Z_{jk}^{A_1})+\frac{\varepsilon^2_{x,r}}{4T\gamma_1}\Big)$,
	 \STATE  $\psi^{j}=\psi^{j-1}\exp\Big(c_{2,r}(\frac{(1-\varepsilon_{y,r})Z_{r}}{T}-Y^{A}_{j})+\frac{\varepsilon^2_{y,r}Z_r}{4T\bar{w}}\Big)$
	\ENDFOR
	\ENDFOR
	\end{algorithmic}
\end{algorithm}

\begin{restatable}{theorem}{thmunknowndis}\label{thm:4}
Under \cref{assumption:1}-\ref{assumption:3}, if $\varepsilon> 0$ $\tau_1=\frac{\sqrt{\varepsilon}}{1-\sqrt{\varepsilon}}$ such that $\tau_1+\varepsilon\le\xi^*$ and $\gamma_{1}=\max\big(\frac{\bar{a}_{k}}{U_{k}}, \frac{\bar{a}_{k}}{(1-\varepsilon)T\bar{a}_{k}-L_{k}}, \frac{\bar{w}}{W_{\varepsilon+\tau_1}}\big)=O\big(\frac{\varepsilon^{2}}{\ln(K/\varepsilon)}\big)$, Algorithm $A_1$ defined in \cref{alg:3} achieves an objective value of at least $\big(1-O(\frac{\varepsilon}{\xi^*-\varepsilon})\big)W_0$ and satisfies the constraints w.p. $1-\varepsilon$.
\end{restatable}

The proof of the theorem is deferred to \cref{appendix:thm4}.

\textbf{Remark}: 
\begin{enumerate}
    \item It can be shown that the problem dependent parameter $\xi^*$ restricts the capacity of Algorithm $A_1$ by affecting $\varepsilon,\;\tau_1$, and $\gamma_1$. This is consistent with our intuition that the problem with two-sided constraints becomes harder if $\xi^*$ decreases.
    \item It is worth noting that the competitive ratio obtained in this paper reflects the high-probability performance of the proposed algorithms under a finite $T$, so it is difficult to compare with the regret bound for dual-mirror-descent methods rigorously due to the different settings. But if $T$ is given and the linear growth of lower and upper bound is assumed, i.e. $L_k$ and $U_k$ are both $O(T)$, $\xi^*$ turns to be a problem-dependent constant, which leads the proposed algorithm towards an $\widetilde{O}(\sqrt{T\ln{(KT)}})$ regret w.h.p., where we hide the potential $\log\log$-term in $\widetilde{O}(\cdot)$. In real-world applications with massive constraints, such as guaranteed advertising delivery \citep{zhang2020request} with hundreds of thousands of ad providers, the proposed algorithm could exceed existing $O(\sqrt{KT})$ averaged regrets \citep{balseiro2020dual,balseiro2021regularized, lobos2021joint} with respect to the number of constraints $K$. More comparisons can also be found in \citep{devanur2019near}. 
    \item Although an LP is solved at the beginning of each stage, according to the proof of \cref{thm:5}, we can use the revenue obtained from previous stage as a good approximation of $W^{r-1}$ before \cref{alg: obj estimator} returning it in practice for $r=1,\ldots, l-1$. Thus, the online property of \cref{alg:3} will not be hurt severely.
\end{enumerate}
 However, sometimes $\xi^*$ is inaccessible in practice when \cref{alg:3} is initialized. We will address the unknown measurement $\xi^*$ in the next section.

\section{Exploring The Measure of Feasibility}\label{sec:feasibility}
As shown in previous sections, the constant $\xi^*$ measures the feasibility of the original problem, and plays a central role in Algorithm $A_{1}$. $\xi^*$ can hardly be obtained in practice. Motivated by \cref{alg: obj estimator}, we propose \cref{alg:M} for estimating $\xi^*$ to provide a more complete analysis.
In this section, we investigate the strong feasible condition, i.e., Assumption~\ref{assumption:3}. Recalling the definition of $E(\tau)$ in \cref{sec:competitive alg}, we can conclude the following results.
\begin{proposition}
 \label{thm:a2}According to the definition of $\xi^{*}$ in LP~\eqref{equ:5}, we have 
\begin{enumerate}
\vspace{-6pt}
    \item When $\xi^{*}\ge 0$, the problem $E(0)$ is feasible. Otherwise, it is infeasible.
    \item When $\xi^{*}>0$ and $0<\xi\le\xi^{*}$, the problem $E(\xi)$ is feasible. We say the problem $E(0)$ satisfies the strong feasible condition with parameter $\xi^*$.
\end{enumerate}
\end{proposition}
We omit the proof of the proposition for simplicity. From \cref{thm:a2}, with the knowledge of $\xi^{*}$, we could easily check the feasibility of $E(0)$. Meanwhile, we are safe to set parameter $0<\xi\leq \xi^*$ if the measure of feasibility $\xi^{*}>0$. Next, we show how to estimate the $\xi^{*}$ from served requests in Algorithm~\ref{alg:M}.
\begin{algorithm}[t]\small
	\caption{Feas\_Estimator$(\{\mathrm{request}_j\}_{r},\; t_{r},\; \gamma_2,\; \delta)$}
	\begin{algorithmic}[1]\label{alg:M}
	\STATE {\bf Input:} requests from $r$-th stage $\{\mathrm{request}_j\}_{r}$, request number $t_r$, relative error $\varepsilon_{x,r}$, problem dependent quantity $\gamma_2$, failure probability $\delta$.
	\STATE {\bf Output:} $\xih=\xi_{\mathrm{max}}-2\varepsilon_{x,r}$
    \STATE Compute $\varepsilon_{x,r}=\sqrt{\frac{4\gamma_{2} T \ln(\frac{K}{\delta})}{t_r}}$
	\STATE Set 
	\begin{equation}\label{equ:47}
	\begin{aligned}
		&\xi_{\mathrm{max}}=\max_{\xi,x}\ \xi\\
        s.t.\; &\frac{t_r}{T}(L_{k}+\xi T\bar{a}_{k})\le\sum_{i\in\mathcal{I}}\sum_{j=t_r+1}^{t_{r+1}} a_{ijk}x_{ij}\le \frac{t_r}{T}U_{k}, \;\forall k\in\mathcal{K}\\
        &\sum_{i\in\mathcal{I}} x_{ij}\le 1, \forall j=t_r+1,\ldots,t_{r+1}\\
        &x_{ij}\ge 0,\; \forall i\in\mathcal{I},\; j=t_r+1,\ldots,t_{r+1}
	\end{aligned}
\end{equation} 	
\end{algorithmic}
\end{algorithm}

\begin{restatable}{theorem}{thmknownth}\label{thm:6}
Under \cref{assumption:1}-\ref{assumption:3}, if  $\gamma_{2}=\max(\frac{\bar{a}_{k}}{U_{k}}, \frac{\bar{a}_{k}}{T\bar{a}_{k}-L_{k}})=O(\frac{\epsilon^{2}}{\ln (K/\varepsilon)})$, \cref{alg:M} with $t_r$ i.i.d. requests outputs $\widehat{\xi} $ such that
\begin{displaymath}
\widehat{\xi} \in[\xi^{*}-4\epsilon_{x,r},\xi^{*}]
\end{displaymath} w.p. $1-2\delta$, where $\epsilon_{x,r}=\sqrt{\frac{4\gamma_{2} T \ln(K/\delta)}{t_{r}}}$.
\end{restatable}

\begin{algorithm}[t]
\setlength{\textfloatsep}{0pt}
	\caption{Algorithm $A_{2}$}
\begin{algorithmic}[1]\label{alg:5}
    \STATE{\bf Input:} $\varepsilon$, $L_k$, $U_k$, $\gamma_1$, $\gamma_2=O\big(\frac{\epsilon^{2}}{\ln (K/\varepsilon)}\big)$, $\bar{a}_{k}$\\
    \STATE{\bf Output:} $\{x_{ij}\}_{i\in\mathcal{K},j\in[T]}$
\STATE Set $l=\log_{2}(\frac{1}{\varepsilon})$, $t_{r}=\varepsilon2^{r}T$, $t_{-1}=\varepsilon T$ and $\delta=\frac{\varepsilon}{3l+2}$
\STATE Set $\xih_{0}$=Feas\_Estimator($\{\mathrm{request}_j\}_{-1},\; t_{-1},\;\gamma_2,\;\delta$)
\STATE Execute the Algorithm $A_{1}(\epsilon,\;\gamma_{1},\;\xih_{0})$
\end{algorithmic}
\setlength{\textfloatsep}{0pt}
\end{algorithm}
The proof is deferred to \cref{appendix:thm6}. 

Now $\widehat{\xi}_{r}$ can be viewed as a good estimate for $\xi^{*}$ from stage 0. Based on the estimation of $\xi^{*}$
, we propose an algorithm $A_{2}$ in \cref{alg:5}. Like the previous methods, we geometrically divide $T$ requests into $l+1$ stages for $r=-1,0,\ldots, l-1$, where the initial stage $r=-1$ and the first stage $r=0$ have $\varepsilon T$ requests. Besides the estimate $W^{-1}$, the estimate $\widehat{\xi}_{0}$ for $\xi^{*}$ is obtained by \cref{alg:M} at the end of the initial stage $r=-1$. Then we consider the new expected problem $E\big(\widehat{\xi}_{0}\big)$ and run Algorithm $A_{1}$ defined in Algorithm~\ref{alg:3} for the rest requests. 

\begin{theorem}\label{thm:7} Under \cref{assumption:1}-\ref{assumption:3}, if $\varepsilon> 0$, $\tau_1=\frac{\sqrt{\varepsilon}}{1-\sqrt{\varepsilon}}$ such that  $\tau_{1}+4\sqrt{\varepsilon}+\varepsilon\le\xi^{*}$ and $\gamma_{1}= \max\big(\frac{\bar{a}_{k}}{U_{k}}, \frac{\bar{a}_{k}}{(1-\epsilon)T\bar{a}_{k}-L_{k}}, \frac{\bar{w}}{W_{\epsilon+\tau_{1}}}\big)=O\big(\frac{\varepsilon^{2}}{\ln (K/\varepsilon)}\big)$,  \cref{alg:5} achieves an objective value of at least $\big(1-O(\frac{\varepsilon}{\xi^{*}-4\sqrt{\varepsilon}-\epsilon})\big)W_0$ and  satisfies the constraints, w.p. $1-\varepsilon$.
\end{theorem}
The proof can be found in \cref{appendix:g}.

Until now, we have dropped the knowledge of the entire distribution $\mathcal{P}$, the objective value $W_\tau$ for problem $E(\tau)$ and the strong feasible constant $\xi^*$ step by step. In practice, the only hyperparameter $\epsilon$ can be well approximated by solving a polynomial approximation of the transcendental equation between $\gamma_1$ and $\epsilon$ in \cref{thm:7}. We regard the proposed \cref{alg:5} as a practical algorithm for two-sided constrained online resource allocation problems. 

\section{Conclusion}
In this paper, we have developed a method for online allocation problems with two-sided resource constraints, which has a wide range of real-world applications. By designing a factor-revealing linear fractional programming, a measure of feasibility $\xi^*$ is defined to facilitate our theoretical analysis. We prove that \cref{alg:3} holds a nearly optimal competitive ratio if the measurement is known and large enough compared with the error parameter, i.e. $\xi^*\gg \varepsilon$. An estimator is also presented in the paper for the unknown $\xi^*$ scenario. We will investigate more efficient extensions of this work in the near future.

\bibliography{all}
\bibliographystyle{unsrtnat}



\newpage
\appendix
\onecolumn

\setlist[enumerate]{itemsep=6pt}
\section{The Examples of Resource Lower Bounds}\label{appendix:a}
In this section, we give three practical instances showing the necessity of lower bound constraints in real-world applications.

\begin{enumerate}
    \item \textbf{Guaranteed Advertising Delivery}: In the online advertising scenario, the advertising publishers will sell the ad impressions in advance with the promise to provide each advertiser an agree-on number of target impressions over a fixed future time period, which is usually written in the contracts. Furthermore, the advertising platform  considers other constraints, such as advertisers' budgets and impressions inventories, and simultaneously maximizes multiple accumulative objectives regarding different interested parties, e.g. Gross Merchandise Volume (GMV) for ad providers and publisher's revenue. This widely used guaranteed delivery advertising model is generally formulated as an online resource allocation problem with two-sided constraints \citep{zhang2020request}.   

    \item[2.1.] \textbf{Fair Channel Constraints}: We consider the online orders assignment in an e-commerce platform, where the platform allocates the orders (or called packages in applications) to different warehouses/logistics providers. It can be shown that when providers are concerned about fairness, the platform can use a simple wholesale price above its marginal cost to coordinate this channel in terms of both achieving the maximum channel profit and attaining the maximum channel utility \citep{haitao2007fairness}. Besides, according to the governmental regulations and contracts between platforms and providers, we usually tend to set lower bounds to the daily accepted orders for special channels, such as new/small-scale providers or those in developing areas.

    \item[2.2.] \textbf{Timeliness Achievement Constraints}: The delivery time for orders is highly related with the customers' shopping experience. Thus, the online shopping platforms also take the time-effectiveness of parcel shipment into account. For every parcel, platforms usually use the timeline achievement rate $r_{u}\in[0,1]$ to denote the probability of arriving at the destination in required $u$ days if we assign this order to one channel, which could be estimated by the historical data and features of order and channel/logistics providers. As we know, long delivery time will impair the consumer’s shopping experience, but reducing delivery time means increasing cost. In order to balance the customers' shopping experience and transportation costs, platforms always set a predefined lower threshold to the timeline achievement rate, which can be modeled as lower bound requirements in order assignment.
\end{enumerate}

\section{Useful Concentrations} \label{appendix:concentration}
In this section, we present the classical concentration inequalities for completeness.

\begin{lemma}\citep{bernstein1946theory}\label{lemma:2}\
\begin{enumerate}
    \item[i.] Suppose that $|X|\le c$ and $\E[X]=0$. For any $t>0$, 
    \begin{displaymath}
     \E\big[\exp(tX)\big]\le \exp\bigg(\frac{\sigma^2}{c^2}\big(e^{tc}-1-ct\big)\bigg)
    \end{displaymath}where $\sigma^2=Var(X)$.
    
    \item[ii.] If $X_1,X_2,\ldots,X_n$ are independent r.v., $\E[X_{i}]=\mu$  and $\mathrm{P}(|X_{i} - \mu|\le c)=1$ ,$\forall i=1,\ldots,n$, then $\forall \varepsilon>0$ following inequality holds
\begin{displaymath}
 \mathrm{P}\bigg(\Big|\frac{\sum_{i=1}^n X_i}{n}-\mu\Big|\ge\varepsilon\bigg)\le2 \exp\bigg(-\frac{n\varepsilon^2}{2\sigma^2+\frac{2c\varepsilon}{3}}\bigg)
\end{displaymath} where $\sigma^2=\frac{\sum_{i=1}^{n}Var(X_i)}{n}$.
\end{enumerate} 
\end{lemma}

The first result, \cref{lemma:2}.i., is a well-known intermediate result of Bennett's inequality. We go a few steps further to construct the algorithm.

\section{Proof for \cref{lemma:1}}\label{appendix:lem1}
\objlemma*
\begin{proof}: 
We consider two linear programming problem, the linear relaxation of sampled integer linear programming \eqref{lp:1}, i.e.
\begin{equation}\tag{S.I I}\label{equ:3}
\begin{aligned}
    &\max_{x}\ \sum_{i\in\mathcal{I},j\in[T]} w_{ij}x_{ij}\\
    s.t.\ &L_{k}\le\sum_{i\in\mathcal{I},j\in[T]} a_{ijk}x_{ij}\le U_{k}, \forall k\in\mathcal{K}\\
    &\sum_{i\in\mathcal{I}} x_{ij}\le 1, \forall j\in[T]\\
    &x_{ij}\ge 0, \forall i\in\mathcal{I}, j\in[T]
\end{aligned}
\end{equation}
and the linear programming that considers the samples from the type of requests perspective, i.e.
\begin{equation}\tag{S.I II}\label{equ:4}
\begin{aligned}
    &\max_{x}\sum_{ 
        i\in\mathcal{I},j\in\mathcal{J}
    } \big|\{j'|j'\in[T],j'=j\}\big|w_{ij}x_{ij}\\
    s.t.\ &L_{k}\le\sum_{i\in\mathcal{I},j\in\mathcal{J}} \big|\{j'|j'\in[T],j'=j\}\big|a_{ijk}x_{ij}\le U_{k}, \forall k\in\mathcal{K}\\
    &\sum_{i\in\mathcal{I}} x_{ij}\le 1, \forall j\in\mathcal{J}\\
    &x_{ij}\ge 0, \forall i\in\mathcal{I}, j\in\mathcal{J}
\end{aligned}
\end{equation}
where $|\cdot|$ denote the cardinality of a given set.We use the $W_{R_1}$ and $W_{R_2}$ to denote the optimal value of Sample Instance~\eqref{equ:3} and \eqref{equ:4}) respectively. Because the LP~\eqref{equ:3} is a relaxation version of ILP~\eqref{lp:1}, we know that $W_{R_1}\ge W_{R}$. For any optimal solution $\{x^{*}_{ij}$ $\forall i\in\mathcal{I}, j\in[T]\}$ of LP~\eqref{equ:3}, it's easy to verify that the solution $\big\{x_{ij}|x_{ij}=\frac{\sum_{j'=j,j'\in[T]}x^{*}_{ij'}}{\big|\{j'|j'\in[T]: j'=j\}\big|},\forall i\in\mathcal{I}, j\in\mathcal{J}\big\}$ is feasible for LP~\eqref{equ:4}, so that $W_{R_2}\ge W_{R_1}$. Moreover, the average of optimal solution for all possible LP~\eqref{equ:4} is a feasible solution for LP~\eqref{equ:E_beta}, $\beta=0$, whose optimal solution is $W_0$. Thus $W_0\ge \mathbb{E}[W_{R_2}]\ge \mathbb{E}[W_{R_1}]\ge \mathbb{E}[W_{R}]$.
\end{proof}

\section{Proof of \cref{lemma:3}}\label{appendix::factor_proof}
\factorlemma*
\begin{proof}
The dual problem of the expected instance, i.e., problem $E(0)$, is
\begin{equation}\label{equ:11}
	\begin{aligned}
		\min_{\alpha, \beta, \rho}\ &\sum_{k\in\mathcal{K}} \alpha_{k}U_{k}-\sum_{k\in\mathcal{K}}\beta_{k}L_{k}+\sum_{j\in\mathcal{J}}\rho_j \\
		s.t.\ &\sum_{k\in\mathcal{K}}(\alpha_{k}-\beta_{k})T p_{j}a_{ijk}-Tp_{j}w_{ij}+\rho_{j}\ge0,\\
		&\;\forall i\in\mathcal{I}, j\in\mathcal{J},\\
		&\alpha_{k},\beta_{k},\rho_{j}\ge 0,k\in\mathcal{K},   j\in\mathcal{J},
	\end{aligned}
\end{equation}

and the dual problem of $E(\tau)$ is
\begin{equation}\label{equ:12}
	\begin{aligned}
		\min_{\alpha, \beta, \rho}\ &\sum_{k\in\mathcal{K}} \alpha_{k}U_{k}-\sum_{k\in\mathcal{K}}\beta_{k}(L_{k}+\tau T\bar{a}_{k})+\sum_{j\in\mathcal{J}}\rho_j \\
		s.t.\ &\sum_{k\in\mathcal{K}}(\alpha_{k}-\beta_{k})T p_{j}a_{ijk}-Tp_{j}w_{ij}+\rho_{j}\ge0,\\
		&\;\forall i\in\mathcal{I}, j\in\mathcal{J},\\
		&\alpha_{k},\beta_{k},\rho_{j}\ge 0,k\in\mathcal{K}, j\in\mathcal{J}.
	\end{aligned}
\end{equation}

It can be observed that LP~\eqref{equ:11} and LP~\eqref{equ:12} share the same feasible set, while LP~\eqref{equ:12} has an extra term $-\sum_{k\in\mathcal{K}}\tau T\bar{a}_{k}\beta_{k}$ in the objective. It is necessary to study the relationship between $\sum_{k\in\mathcal{K}}\tau T\bar{a}_{k}\beta_{k}$ and $\sum_{k\in\mathcal{K}} \alpha_{k}U_{k}-\sum_{k\in\mathcal{K}}\beta_{k}L_{k}+\sum_{j\in\mathcal{J}}\rho_j$ under the dual constraints if we want to obtain the ratio of $W_{\tau}$ to $W_0$. Motivated by \citep{jain2003greedy}, we propose a Factor-Revealing Linear Programming method for this analysis.

In order to derive the competitive ratio of the cumulative revenue obtained by the Algorithm $\widetilde{P}$ to $W_0$, we need to find a number $c\in(0,1)$ which makes $W_{\tau}\ge(1-c)W_0$ always hold. Considering the dual LP~\eqref{equ:11} and LP~\eqref{equ:12}, if we can show that $\sum_{k\in\mathcal{K}}\frac{\varepsilon}{1-\varepsilon}T\bar{a}_{k}\beta_{k}\le c (\sum_{k\in\mathcal{K}} \alpha_{k}U_{k}-\sum_{k\in\mathcal{K}}\beta_{k}L_{k}+\sum_{j\in\mathcal{J}}\rho_j)$ for any dual feasible solution, it will give us that $W_{\tau}\ge(1-c)W_0$. Hence this question can be translated to solving the following linear fractional programming
\begin{equation}\label{equ:13}
	\begin{aligned}
	\max_{\alpha, \beta, \rho}\ &\frac{\sum_{k\in\mathcal{K}}\frac{\varepsilon}{1-\varepsilon}T\bar{a}_{k}\beta_{k}}{\sum_{k\in\mathcal{K}} \alpha_{k}U_{k}-\sum_{k\in\mathcal{K}}\beta_{k}L_{k}+\sum_{j\in\mathcal{J}}\rho_j} \\
		s.t.\ &\sum_{k\in\mathcal{K}}(\alpha_{k}-\beta_{k})T p_{j}a_{ijk}-Tp_{j}w_{ij}+\rho_{j}\ge0,\\
		&\forall i\in\mathcal{I},j\in\mathcal{J}\\
		&\alpha_{k},\beta_{k},\rho_{j}\ge 0,k\in\mathcal{K}, j\in\mathcal{J}.
	\end{aligned}
\end{equation}
Since $\sum_{k\in\mathcal{K}} \alpha_{k}U_{k}-\sum_{k\in\mathcal{K}}\beta_{k}L_{k}+\sum_{j\in\mathcal{J}}\rho_j\ge W_0 > 0$ for any dual feasible solution, we can do the following transformation
\begin{equation*}
    \begin{aligned}
         &\widetilde{\alpha}_{k}=\frac{\alpha_{k}}{\sum_{k\in\mathcal{K}} \alpha_{k}U_{k}-\sum_{k\in\mathcal{K}}\beta_{k}L_{k}+\sum_{j\in\mathcal{J}}\rho_j}\\
    \end{aligned}
\end{equation*}
\begin{align*}
    &\widetilde{\beta}_{k}=\frac{\beta_{k}}{\sum_{k\in\mathcal{K}} \alpha_{k}U_{k}-\sum_{k\in\mathcal{K}}\beta_{k}L_{k}+\sum_{j\in\mathcal{J}}\rho_j}\\
    &\widetilde{\rho}_{j}=\frac{\rho_{j}}{\sum_{k\in\mathcal{K}} \alpha_{k}U_{k}-\sum_{k\in\mathcal{K}}\beta_{k}L_{k}+\sum_{j\in\mathcal{J}}\rho_j} \\
    & z=\frac{1}{\sum_{k\in\mathcal{K}} \alpha_{k}U_{k}-\sum_{k\in\mathcal{K}}\beta_{k}L_{k}+\sum_{j\in\mathcal{J}}\rho_j}.
\end{align*}

In this way, we transfer the linear fractional programming~\eqref{equ:13} into 
\begin{equation}\label{equ:14}
	\begin{aligned}
		\max_{\widetilde{\alpha}, \widetilde{\beta}, \widetilde{\rho}}\ &\sum_{k\in\mathcal{K}}\frac{\varepsilon}{1-\varepsilon}T\bar{a}_{k}\widetilde{\beta}_{k}\\
		s.t.\ &\sum_{k\in\mathcal{K}}(\widetilde{\alpha}_{k}-\widetilde{\beta}_{k})T p_{j}a_{ijk}-Tp_{j}w_{ij}z+\widetilde{\rho}_{j}\ge0,\\ &\forall i\in\mathcal{I},j\in\mathcal{J}\\
		&\sum_{k\in\mathcal{K}}\widetilde{\alpha}_{k}U_{k}-\sum_{k\in\mathcal{K}}\widetilde{\beta}_{k}L_{k}+\sum_{j\in\mathcal{J}}\widetilde{\rho}_j=1\\
		&\widetilde{\alpha}_{k},\widetilde{\beta}_{k},\widetilde{\rho}_{j},z\ge 0,k\in\mathcal{K}, j\in\mathcal{J}.
	\end{aligned}
\end{equation}
We investigate the dual problem of LP~\eqref{equ:14} as follows
\begin{equation}\label{equ:15}
	\begin{aligned}
		\min_{t,d}\ &t \\
		s.t.\ &\sum_{i\in\mathcal{I}}d_{ij}\le t\\
        &\sum_{i\in\mathcal{I},j\in\mathcal{J}} d_{ij}T p_{j}a_{ijk}\le t U_{k}\\
        &\sum_{i\in\mathcal{I},j\in\mathcal{J}} d_{ij}T p_{j}a_{ijk}\ge tL_{k}+\frac{\varepsilon}{1-\varepsilon}T\bar{a}_{k}\\
        &\forall d_{ij}\ge0,t\in\R,\forall i\in\mathcal{I}, j\in\mathcal{J}.
	\end{aligned}
\end{equation}

By the first constraint of LP~\eqref{equ:15}, it can be observed that $t\ge 0$ holds. With the auxiliary variable $ z_{ij}$ which makes $d_{ij} = t z_{ij}$,we reformulate LP~\eqref{equ:15} into
 \begin{equation}\label{equ:16}
 	\begin{aligned}
 		t^{*}=\ &\min_{t,z} t \\
		s.t.\  &t\bigg(\sum_{i\in\mathcal{I}}z_{ij}-1\bigg)\le 0\\
         &t\bigg(\sum_{i\in\mathcal{I},j\in\mathcal{J}} z_{ij}T p_{j}a_{ijk}-U_{k}\bigg)\le 0\\
        &t\bigg(\sum_{i\in\mathcal{I},j\in\mathcal{J}}z_{ij}T p_{j}a_{ijk}-L_{k}\bigg)\ge\frac{\varepsilon}{1-\varepsilon}T\bar{a}_{k}\\
        &\forall z_{ij}\ge 0, t\ge 0,\forall i\in\mathcal{I}, j\in\mathcal{J}.
  \end{aligned}
 \end{equation}
According to the strong feasible condition \cref{assumption:3}, there exists a feasible solution $\{x_{ij}'|x_{ij}'\ge 0\}$ to LP~\eqref{equ:5} such that $\sum_{i\in\mathcal{I},j\in\mathcal{J}} x_{ij}'T p_{j}a_{ijk}\le U_{k}$, $\sum_{i\in\mathcal{I},j\in\mathcal{J}} Tp_{j}a_{ijk}x_{ij}'\ge L_{k}+\xi^* Ta_{k}$ and $\sum_{i\in\mathcal{I}}x_{ij}'\le 1$, $\forall k\in \mathcal{K}$. Therefore,  $t=\frac{\tau}{\xi^*}$ and $z_{ij}=x_{ij}'\ \forall i\in\mathcal{I},\ j\in\mathcal{J}$ is a feasible solution to LP~\eqref{equ:16}. As a result, we have $t^{*}\le\frac{\tau}{\xi^*}$. Since LP~\eqref{equ:13} has the same optimum as LP~\eqref{equ:16}, we have that 
\begin{equation*}
    \begin{aligned}
         &\sum_{k\in\mathcal{K}}\alpha_{k}U_{k}-\sum_{k\in\mathcal{K}}\beta_{k}\bigg(L_{k}+\frac{\varepsilon}{1-\varepsilon}T\bar{a}_{k}\bigg)+\sum_{j\in\mathcal{J}}\rho_j\\
        &=\sum_{k\in\mathcal{K}} \alpha_{k}U_{k}-\sum_{k\in\mathcal{K}}\beta_{k}L_{k}+\sum_{j\in\mathcal{J}}\rho_j-\sum_{k\in\mathcal{K}}\frac{\varepsilon}{1-\varepsilon}T\bar{a}_{k}\beta_{k}\\
        & \ge\bigg(1-\frac{\tau}{\xi^*}\bigg)\bigg(\sum_{k\in\mathcal{K}} \alpha_{k}U_{k}-\sum_{k\in\mathcal{K}}\beta_{k}L_{k}+\sum_{j\in\mathcal{J}}\rho_j\bigg)
    \end{aligned}
\end{equation*}
under the constraint of LP~\eqref{equ:11}. Therefore, $W_{\tau}\ge(1-\frac{\tau}{\xi^*})W_0$.
\end{proof}
The factor-revealing linear fractional programming analysis shows
the way we develop the definition of $\xi^{*}$ and enlightens
the design of feasibility estimator in Algorithm~\ref{alg:M}. Besides, the proof framework is reused in proving \cref{lemma:5ub}.

\section{Proof of \cref{thm:2} and \cref{thm:1}}\label{appendix:thm1}
We restate the \cref{thm:2} as follows.
\thmkowndis*
\begin{proof}
We first prove that, for every resource $k$, the consumed resource is below the capacity $U_{k}$ w.h.p..
\begin{equation}\label{equ:7}
	\begin{aligned}
		&P\left(\sum_{j=1}^{T}X^{\widetilde{P}}_{jk}\ge U_k\right)\\
		&=P\left(\sum_{j=1}^{T}\left(X^{\widetilde{P}}_{jk}-\mathbb{E}[X^{\widetilde{P}}_{jk}]\right)\ge U_k-T\mathbb{E}[X^{\widetilde{P}}_{jk}]\right)\\
		&\le \exp\left\{-\frac{\left(U_k-T\mathbb{E}[X^{\widetilde{P}}_{jk}]\right)^2}{2T\sigma^2+\frac{2}{3}\bar{a}_k\left(U_k-T\mathbb{E}[X^{\widetilde{P}}_{jk}]\right)}\right\}\\
		&= \exp\left\{-\frac{U_k-T\mathbb{E}[X^{\widetilde{P}}_{jk}]}{2\frac{T\sigma^2}{U_k-T\mathbb{E}[X^{\widetilde{P}}_{jk}]}+\frac{2}{3}\bar{a}_{k}}\right\}
	\end{aligned}
\end{equation}
\begin{align*}
		&\le \exp\left\{-\frac{\varepsilon^2}{2(1-\frac{2}{3}\varepsilon)\frac{\bar{a}_{k}}{U_k}}\right\}\\
		&\le \exp\left\{-\frac{\varepsilon^2}{2(1-\frac{2}{3}\varepsilon)\gamma}\right\}\\
		&\le\frac{\varepsilon}{2K+1}
\end{align*}
where the first equality follows from $\mathbb{E}[X^{\widetilde{P}}_{1j}]=\mathbb{E}[X^{\widetilde{P}}_{2j}]=\dots=\mathbb{E}[X^{\widetilde{P}}_{Tj}]$; the first inequality from \cref{lemma:2} and setting $\sigma^2=Var(X^{\widetilde{P}}_{jk})$; in the second inequality, it can verified that 
$\sigma^2\le\mathbb{E}[(X^{\widetilde{P}}_{jk})^2]\le\bar{a}_{k}\mathbb{E}[X^{\widetilde{P}}_{jk}]\le \frac{(1-\varepsilon)\bar{a}_{k}U_{k}}{T}$, $U_k-T \mathbb{E}[X^{\widetilde{P}}_{jk}]\ge\varepsilon U_k$ and $\frac{T\sigma^2}{U_{k}-T\mathbb{E}[X^{\widetilde{P}}_{ij}]}\le\bar{a}_{k}\frac{1-\varepsilon}{\varepsilon}$ so that $-\frac{U_k-T\mathbb{E}[X^{\widetilde{P}}_{jk}]}{2\frac{T\sigma^2}{U_k-T\mathbb{E}[X^{\widetilde{P}}_{jk}]}+\frac{2}{3}\bar{a}_{k}}\le$ $-\frac{\varepsilon U_{k}}{2\bar{a}_{k}\frac{1-\varepsilon}{\varepsilon}+\frac{2}{3}\bar{a}_{k}}=-\frac{\varepsilon^2}{2(1-\frac{2}{3}\varepsilon)\frac{\bar{a}_{k}}{U_k}}$; the third inequality from $\gamma\ge\frac{\bar{a}_{k}}{U_{k}}$; the final inequality from $\gamma=O(\frac{\varepsilon^2}{\ln(K/\varepsilon)})$.

Next, we verify that the Algorithm $\widetilde{P}$ satisfies the lower resource bound with high probability.
\begin{equation}\label{equ:8}
	\begin{aligned}
		&P\left(\sum_{j=1}^{T}X^{\widetilde{P}}_{jk}\le L_k\right)\\
		&=P\left(\sum_{j=1}^{T}\left(\mathbb{E}[X^{\widetilde{P}}_{jk}]-X_{jk}\right)\ge T\mathbb{E}[X^{\widetilde{P}}_{jk}]-L_k\right)\\
		&\le \exp\left\{-\frac{\left(T\mathbb{E}[X^{\widetilde{P}}_{jk}]-L_k\right)^2}{2T\sigma^2+\frac{2}{3}\bar{a}_k\left(T\mathbb{E}[X^{\widetilde{P}}_{jk}]-L_k\right)}\right\}\\
		&= \exp\left\{-\frac{(T\mathbb{E}[X^{\widetilde{P}}_{jk}]-L_k)}{2\frac{T\sigma^2}{T \mathbb{E}(X_{jk})-L_k}+\frac{2}{3}\bar{a}_{k}}\right\}\\
    	&\le \exp\left\{-\frac{\varepsilon^2}{2(1-\frac{2}{3}\varepsilon)\frac{\bar{a}_{k}}{T\bar{a}_{k}-L_{k}}}\right\}\\
    	&\le\frac{\varepsilon}{2K+1}
    \end{aligned}
\end{equation}
where the first inequality from \cref{lemma:2} and setting $\sigma^2=Var(X^{\widetilde{P}}_{jk})$; in the second inequality, we could verify that
$\sigma^2=Var(X^{\widetilde{P}}_{jk})=Var(\bar{a}_{k}-X^{\widetilde{P}}_{jk})\le\mathbb{E}[(\bar{a}_{k}-X^{\widetilde{P}}_{jk})^{2}]\le\bar{a}_{k}\mathbb{E}[\bar{a}_{k}-X^{\widetilde{P}}_{jk}]\le \frac{(1-\varepsilon)\bar{a}_{k}(T\bar{a}_{k}-L_{k})}{T}$, $T\mathbb{E}[X^{\widetilde{P}}_{jk}]-L_{k}\ge\varepsilon (T\bar{a}_{k}-L_{k})$, and $\frac{T\sigma^2}{T \mathbb{E}[X^{\widetilde{P}}_{jk}]-L_k}\le\bar{a}_{k}\frac{1-\varepsilon}{\varepsilon}$ so that $-\frac{(T\mathbb{E}[X^{\widetilde{P}}_{jk}]-L_k)}{2\frac{T\sigma^2}{T \mathbb{E}[X^{\widetilde{P}}_{jk}]-L_k}+\frac{2}{3}\bar{a}_{k}}\le$ $-\frac{\varepsilon (T\bar{a}_{k}-L_{k})}{2\bar{a}_{k}\frac{1-\varepsilon}{\varepsilon}+\frac{2}{3}\bar{a}_{k}}=-\frac{\varepsilon^2}{2(1-\frac{2}{3}\varepsilon)\frac{\bar{a}_{k}}{T\bar{a}_{k}-L_{k}}}$; the final inequality from $\gamma=O(\frac{\varepsilon^2}{\ln(\frac{K}{\varepsilon})})$.
	
Therefore, from the previous outcomes, the consumed resource $k$ satisfies our lower and upper bound requirements, w.h.p. Next, we investigate the revenue the Algorithm $\widetilde{P}$ brings.
	
\begin{equation}\label{equ:9}
	\begin{aligned}
		&P\left(\sum_{j=1}^{T}Y^{\widetilde{P}}_{j}\le (1-2\varepsilon)W_{\tau}\right)\\
		&=P\left(\sum_{j=1}^{T}\left(\mathbb{E}[Y^{\widetilde{P}}_{j}]-Y^{\widetilde{P}}_{j}\right)\ge T \mathbb{E}[Y^{\widetilde{P}}_{j}]-(1-2\varepsilon)W_{\tau}\right)\\
		&\le \exp\left\{-\frac{\left(T \mathbb{E}[Y^{\widetilde{P}}_{j}]-(1-2\varepsilon)W_{\tau}\right)^2}{2T\sigma_1^2+\frac{2}{3}\bar{w}\left(T\mathbb{E}[Y^{\widetilde{P}}_{j}]-(1-2\varepsilon)W_{\tau}\right)}\right\}
	\end{aligned}
\end{equation} 
\begin{align*}
	&= \exp\left\{-\frac{T\mathbb{E}[Y^{\widetilde{P}}_{j}]-(1-2\varepsilon)W_{\tau}}{2\frac{T\sigma_1^2}{T\mathbb{E}[Y^{\widetilde{P}}_{j}]-(1-2\varepsilon)W_{\tau}}+\frac{2}{3}\bar{w}}\right\}\\
	&\le \exp\left\{-\frac{\varepsilon^2}{2(1-\frac{2}{3}\varepsilon)\frac{\bar{w}}{W_{\tau}}}\right\}\\
	&\le\frac{\varepsilon}{2K+1}
\end{align*}

where the first equality follows from $\mathbb{E}[Y^{\widetilde{P}}_{1}]=\mathbb{E}[Y^{\widetilde{P}}_{2}]=\dots=\mathbb{E}[Y^{\widetilde{P}}_{T}]$; the first inequality from \cref{lemma:2} and setting $\sigma_{1}^2=Var(Y^{\widetilde{P}}_{j})$; in the second inequality, we could easily verify that 
$\sigma_{1}^2\le\mathbb{E}[(Y^{\widetilde{P}}_{i})^2]\le\bar{w}\mathbb{E}[Y^{\widetilde{P}}_{j}]\le \frac{(1-\varepsilon)\bar{w} W_{\tau}}{T}$ and $\mathbb{E}[Y^{\widetilde{P}}_{j}]=(1-\varepsilon)\frac{W_{\tau}}{T}$ so that $-\frac{T\mathbb{E}[Y^{\widetilde{P}}_{j}]-(1-2\varepsilon)W_{\tau}}{2\frac{T\sigma_1^2}{T\mathbb{E}[Y^{\widetilde{P}}_{j}]-(1-2\varepsilon)W_{\tau}}+\frac{2}{3}\bar{w}}\le$ $-\frac{\varepsilon W_{\tau}}{2\bar{w}\frac{1-\varepsilon}{\varepsilon}+\frac{2}{3}\bar{w}}=-\frac{\varepsilon^2}{2(1-\frac{2}{3}\varepsilon)\frac{\bar{w}}{W_{\tau}}}$; the final inequality follows from $\gamma=O(\frac{\varepsilon^2}{\ln(\frac{K}{\varepsilon})})$.

From equation~\eqref{equ:7}-\eqref{equ:9},
\begin{equation}\label{equ:10}
	\begin{aligned}
		P(\sum_{j=1}^{T}Y^{\widetilde{P}}_{j}\le (1-2\varepsilon)W_{\tau})+&\sum_{k\in\mathcal{K}}P(\sum_{j=1}^{T}X^{\widetilde{P}}_{jk}\notin[L_k,U_k])\\
		&\le (2K+1)\frac{\varepsilon}{2K+1}\le\varepsilon
	\end{aligned} 
\end{equation}when $\gamma=O(\frac{\varepsilon^2}{\ln(\frac{K}{\varepsilon})})$, where $\gamma=\max(\frac{\bar{w}}{W_{\tau}},\frac{\bar{a}_{k}}{T\bar{a}_{k}-L_k},\frac{\bar{a}_{k}}{U_k})$.
\end{proof}
\thmknowndis*

With \cref{assumption:3}, \cref{lemma:3} and \cref{thm:2}, the cumulative revenue is larger than $(1-2\varepsilon)W_{\tau}\ge(1-2\varepsilon)(1-\frac{\tau}{\xi^*})W_0\ge(1-(2+\frac{1}{\xi^*})\varepsilon)W_0$. We finish the proof of \cref{thm:1}.

\section{Proof of \cref{thm:3} }\label{appendix:thm3}
\thmknownobj*
\begin{proof}
We consider the good event defined by 
\begin{align*}
    G \coloneqq& \left\{ \sum_{j=1}^{s}X^{A}_{jk}+\sum_{j=s+1}^{T}X^{\widetilde{P}}_{jk}\le U_k, \forall k \in \mathcal{K} \right\}\cap \left\{ \sum_{j=1}^{s}X^{A}_{jk}+\sum_{j=s+1}^{T}X^{\widetilde{P}}_{jk}\ge L_k, \forall k \in \mathcal{K} \right\} \\
    &\cap \left\{ \sum_{j=1}^{s}Y^{A}_{j}+\sum_{j=s+1}^{T}Y^{\widetilde{P}}_{j}\ge (1-2\varepsilon)W_{\tau}\right\}\\
     =& G_1\cap G_2 \cap G_3,
\end{align*}
which means the hybrid Algorithm $A^s\widetilde{P}^{T-s}$ can achieve at least  $(1-2\varepsilon)W_{\tau}$ revenue while satisfying the two-side constraints. We will show that the probability of the complement event $G^{c}$ can be bounded by some moment generating functions.

For the first bad event $G_1^{c}$, we have  
\begin{equation}\label{equ:18}
	\begin{aligned}
		&P(\sum_{j=1}^{s}X^{A}_{jk}+\sum_{j=s+1}^{T}X^{\widetilde{P}}_{jk}\ge U_k)\\
		&\le \min_{t>0}\E\left[\exp(t(\sum_{j=1}^{s}X^{A}_{jk}+\sum_{j=s+1}^{T}X^{\widetilde{P}}_{jk}-U_k))\right]\\
		&= \min_{t>0}\E\left[\exp(t(\sum_{j=1}^{s}X^{A}_{jk}-\frac{s}{T}U_{k})+t(\sum_{j=s+1}^{T}X^{\widetilde{P}}_{jk}-\frac{T-s}{T}U_k))\right]\\
		&= \min_{t>0}\mathbb{E}\left[\phi_k^s(t)\exp(t(\sum_{j=s+1}^{T}(X^{\widetilde{P}}_{jk}-\mathbb{E}[X^{\widetilde{P}}_{jk}]))+\frac{T-s}{T}t(T \mathbb{E}[X^{\widetilde{P}}_{jk}]-U_k))\right] \\
		&\le \min_{t>0}\mathbb{E}\left[\phi_k^s(t)\exp((T-s)\frac{\sigma^2}{\bar{a}_k^2}(e^{t\bar{a}_{k}}-1-t\bar{a}_k)+\frac{-(T-s)t\varepsilon U_k}{T})\right]\\
		&\le \min_{t>0}\mathbb{E}\left[\phi_k^s(t)\exp(\frac{T-s}{T}\frac{(1-\varepsilon)U_{k}}{\bar{a}_k}(e^{t\bar{a}_{k}}-1-t\bar{a}_k-t\frac{\varepsilon}{1-\varepsilon}\bar{a}_{k}))\right]\\
		&\le \mathbb{E}\left[\phi_k^s(\frac{-\ln(1-\varepsilon)}{\bar{a}_{k}})\exp(-\frac{(1-\varepsilon)(T-s)U_k}{T\bar{a}_{k}}((1+\eta)\ln(1+\eta)-\eta))\right]\\
		&\le \mathbb{E}\left[\phi_k^s(\frac{-\ln(1-\varepsilon)}{\bar{a}_{k}})\exp(-\frac{T-s}{T}\frac{\varepsilon^2}{2\gamma(1-\frac{2}{3}\varepsilon)})\right]
	\end{aligned}
	\end{equation} where the first inequality follows from $\exp(t(\sum_{j=1}^{s}X^{A}_{jk}+\sum_{j=s+1}^{T}X^{\widetilde{P}}_{jk}-U_k))\ge 1$ when $\sum_{j=1}^{s}X^{A_{j}}_{jk}+\sum_{j=s+1}^{T}X^{\widetilde{P}}_{jk}\ge U_k$; in the second equality, we set $\phi_k^s(t)=\exp(t(\sum_{j=1}^{s}X^{A_{j}}_{jk}-\frac{s}{T}U_{k}))$; the second inequality from \cref{lemma:2} and $T\E[X_{jk}^{\widetilde{P}}]\le(1-\varepsilon)U_{k}$; the third inequality from $\sigma^2=Var(X^{\widetilde{P}}_{jk})\le \frac{(1-\varepsilon)\bar{a}_{k}U_{k}}{T}$; in the fourth  inequality, we set $t=\frac{-\ln(1-\varepsilon)}{\bar{a}_{k}},\eta=\frac{\varepsilon}{1-\varepsilon}$; 
    then the last inequality from $(1+\eta)\ln(1+\eta)-\eta\ge\frac{\eta^2}{2+\frac{2}{3}\eta}$ and the definition   of $\gamma$.

Next,for the bad event $G_2^{c}$, we have 
\begin{equation}\label{equ:19}
	\begin{aligned}
		&P(\sum_{j=1}^{s}X^{A}_{jk}+\sum_{j=s+1}^{T}X^{\widetilde{P}}_{jk}\le L_k)\\
		&\le \min_{t>0}\mathbb{E}\left[\exp(t(L_k-\sum_{j=1}^{s}X^{A}_{jk}-\sum_{j=s+1}^{T}X^{\widetilde{P}}_{jk}))\right]\\
	&= \min_{t>0}\mathbb{E}\left[\exp(t(\frac{s}{T}L_{k}-\sum_{j=1}^{s}X^{A}_{jk})+t(\frac{T-s}{T}L_{k}-\sum_{j=s+1}^{n}X^{\widetilde{P}}_{jk}))\right]\\
	\end{aligned}
\end{equation}	
\begin{align*}
	&= \min_{t>0}\mathbb{E}\left[\varphi_k^s(t)\exp(t\sum_{j=s+1}^{T}(\mathbb{E}[X^{\widetilde{P}}_{jk}]-X^{\widetilde{P}}_{jk})+\frac{T-s}{T}t(L_k-T\mathbb{E}[X^{\widetilde{P}}_{jk}]))\right] \\
    &\le \min_{t>0}\mathbb{E}\left[\varphi_k^s(t)\exp((T-s)\frac{\sigma^2}{\bar{a}_{k}^2}(e^{t\bar{a}_{k}}-1-t\bar{a}_{k})-\frac{(T-s)\varepsilon (T\bar{a}_{k}-L_{k}))}{T}t\right]\\
	&\le \min_{t>0}\mathbb{E}\left[\varphi_k^s(t)\exp(\frac{(1-\varepsilon)(T-s)(T\bar{a}_{k}-L_{k})}{\bar{a}_{k}}(e^{t\bar{a}_{k}}-1-t\bar{a}_{k}-t\frac{\varepsilon }{1-\varepsilon}\bar{a}_{k}))\right]\\
	&\le \mathbb{E}\left[\varphi_k^s(\frac{-\ln(1-\varepsilon)}{\bar{a}_k})\exp(-\frac{(1-\varepsilon)(T-s)(T\bar{a}_{k}-L_{k})}{\bar{a}_{k}}((1+\eta)\ln(1+\eta)-\eta))\right]\\
	&\le \mathbb{E}\left[\varphi_k^s(\frac{-\ln(1-\varepsilon)}{\bar{a}_k})\exp(-\frac{T-s}{T}\frac{\varepsilon^2}{2\gamma(1-\frac{2}{3}\varepsilon)})\right]
\end{align*}
 where in the second equality, we set $\varphi_k^s(t)=\exp(t(\frac{s}{T}L_{k}-\sum_{j=1}^{s}X^{A}_{jk}))$; the second inequality from \cref{lemma:2} and $T\E[X_{jk}^{\widetilde{P}}]\ge(1-\varepsilon)L_{k}+\varepsilon T\bar{a}_{k}$; the third inequality from $\sigma^2=Var(X^{\widetilde{P}}_{jk})=Var(\bar{a}_{k}-X^{\widetilde{P}}_{jk})\le \bar{a}_{k}\mathbb{E}[\bar{a}_{k}-X^{\widetilde{P}}_{jk}]\le\bar{a}_{k}\frac{(1-\varepsilon)(T\bar{a}_{k}-L_{k})}{T}$; in the fourth  inequality, we set $t=\frac{-\ln(1-\varepsilon)}{\bar{a}_{k}},\eta=\frac{\varepsilon}{1-\varepsilon}$; The last inequality from $(1+\eta)\ln(1+\eta)-\eta\ge\frac{\eta^2}{2+\frac{2}{3}\eta}$ and the definition of $\gamma$.

\vspace{-3pt}
Finally, we bound the probability of event $G_3^{c}$ by
\begin{equation}\label{equ:20}
	\begin{aligned}
		&P\left(\sum_{j=1}^{s}Y^{A}_{j}+\sum_{j=s+1}^{T}Y^{\widetilde{P}}_{j}\le (1-2\varepsilon)W_{\tau}\right)\\
		&\le \min_{t>0}\mathbb{E}\left[\exp(t((1-2\varepsilon)W_{\tau}-\sum_{j=1}^{s}Y^{A}_{j}-\sum_{j=s+1}^{T}Y^{\widetilde{P}}_{j}))\right]\\
		&= \min_{t>0}\mathbb{E}\left[\exp(t(\frac{s}{T}(1-2\varepsilon)W_{\tau}-\sum_{j=1}^{s}Y^{A}_{j})+t(\frac{T-s}{T}(1-2\varepsilon)W_{\tau}-\sum_{j=s+1}^{T}Y^{\widetilde{P}}_{j}))\right]\\
		&= \min_{t>0}\mathbb{E}\left[\psi^s(t)\exp(t\sum_{j=s+1}^{T}(\mathbb{E}[Y^{\widetilde{P}}_{j}]-Y^{\widetilde{P}}_{j})+\frac{T-s}{T}t((1-2\varepsilon)W_{\tau}-T\mathbb{E}[Y^{\widetilde{P}}_{j}]))\right] \\
		&\le \min_{t>0}\mathbb{E}\left[\psi^s(t)\exp((T-s)\frac{\sigma_1^2}{\bar{w}^2}(e^{t\bar{w}}-1-t\bar{w})+\frac{-(T-s)t\varepsilon W_{\tau}}{T})\right]\\
	&\le \min_{t>0}\mathbb{E}\left[\psi^s(t)\exp(\frac{(1-\varepsilon)(T-s)W_{\tau}}{T\bar{w}}(e^{t\bar{w}}-1-t\bar{w}-\frac{\varepsilon}{1-\varepsilon}t\bar{w}))\right]\\
	&\le \mathbb{E}\left[\psi^s(\frac{-\ln(1-\varepsilon)}{\bar{w}})\exp(-\frac{(1-\varepsilon)(T-s)W_{\tau}}{T\bar{w}}((1+\eta)\ln(1+\eta)-\eta))\right]\\
	&\le \mathbb{E}\left[\psi^s(\frac{-\ln(1-\varepsilon)}{\bar{w}})\exp(-\frac{T-s}{T}\frac{\varepsilon^2}{2(1-\frac{2}{3}\varepsilon)\gamma}) \right]
	\end{aligned}
\end{equation}
where in the second equality, we set $\psi^s(t)=\exp(t(\frac{s}{T}(1-2\varepsilon)W_{\tau}-\sum_{j=1}^{s}Y^{A}_{j}))$; the second inequality from \cref{lemma:2} and $T\E[Y_{j}^{\widetilde{P}}]=(1-\varepsilon)W_{\tau}$; the third inequality from $\sigma_1^2=Var(Y^{\widetilde{P}}_{j})\le\bar{w}\E[Y^{\widetilde{P}}_{j}]\le\frac{(1-\varepsilon)\bar{w}W_{\tau}}{T}$; in the fourth  inequality, we set $t=\frac{-\ln(1-\varepsilon)}{\bar{w}},\eta=\frac{\varepsilon}{1-\varepsilon}$; The 
last inequality from $(1+\eta)\ln(1+\eta)-\eta\ge\frac{\eta^2}{2+\frac{2}{3}\eta}$ and the definition of $\gamma$.

\vspace{-3pt}
With the inequalities~\eqref{equ:18}-\eqref{equ:20} and union bound in probability theory, we can show that $P(G^c) \le \mathcal{F}(A^s\widetilde{P}^{T-s})$ where $\mathcal{F}(A^s\widetilde{P}^{T-s})$ is defined by
\begin{align*}
	\mathcal{F}(A^s\widetilde{P}^{T-s})=&\mathbb{E}\bigg[\sum_{k\in\mathcal{K}}\phi_k^s(\frac{-\ln(1-\varepsilon)}{\bar{a}_k})\exp(-\frac{T-s}{T}\frac{\varepsilon^2}{2(1-\frac{2}{3}\varepsilon)\gamma})+ \sum_{k\in\mathcal{K}}\varphi_k^s(-\frac{\ln(1-\varepsilon)}{\bar{a}_k})\exp(-\frac{T-s}{T}\frac{\varepsilon^2}{2(1-\frac{2}{3}\varepsilon)\gamma})\\
	&+ \psi^s(\frac{-\ln(1-\varepsilon)}{\bar{w}})\exp(-\frac{T-s}{T}\frac{\varepsilon^2}{2(1-\frac{2}{3}\varepsilon)\gamma})\bigg]
\end{align*}

In \cref{thm:2}, we have proven that $\mathcal{F}(\widetilde{P}^{T})=(2K+1)\exp\left(-\frac{\varepsilon^2}{2(1-\frac{2}{3}\varepsilon)\gamma}\right)\le\varepsilon$, and we will show that $\mathcal{F}(A^{s}\widetilde{P}^{T-s})\le\mathcal{F}(A^{s-1}\widetilde{P}^{T-s+1})$ in the \cref{lemma:4}. Thus, we have that $\mathcal{F}(A^{T})\le\mathcal{F}(\widetilde{P}^{T})\le\varepsilon$ by induction. Substituting $\tau = \frac{\varepsilon}{1-\varepsilon}$ and $W_\tau \geq(1-\frac{\tau}{\xi^*})W_0$ in \cref{thm:1}, we complete the proof of \cref{thm:3} .
\end{proof}

\begin{lemma}\label{lemma:4}
$\mathcal{F}(A^{s}\widetilde{P}^{T-s})\le\mathcal{F}(A^{s-1}\widetilde{P}^{T-s+1})$
\end{lemma}
\begin{proof}
By the definition of $\mathcal{F}(A^s\widetilde{P}^{T-s})$, we have that
\begin{equation}\label{equ::22}
    \begin{aligned}
        \mathcal{F}(A^s\widetilde{P}^{T-s}) =& \left(\E\bigg[\sum_{k\in\mathcal{K}}\phi_k^{s-1}(\frac{-\ln(1-\varepsilon)}{\bar{a}_k})\exp((-\frac{\ln(1-\varepsilon)}{\bar{a}_k})(X_{jk}^{A} - \frac{U_{k}}{T}))\right.\\
        &+ \sum_{k\in\mathcal{K}}\varphi_k^{s-1}(-\frac{\ln(1-\varepsilon)}{\bar{a}_k})\exp((-\frac{\ln(1-\varepsilon)}{\bar{a}_k})(\frac{L_k}{T} - X_{jk}^A)) \\
        &\left.+ \psi^{s-1}(\frac{-\ln(1-\varepsilon)}{\bar{w}}))\exp((\frac{-\ln(1-\varepsilon)}{\bar{w}})(\frac{(1 - 2\varepsilon)W_{\tau}}{T} - Y_s^A))\bigg]\right)\exp(-\frac{T-s}{T}\frac{\varepsilon^2}{2(1-\frac{2}{3}\varepsilon)\gamma}).
    \end{aligned}
\end{equation}
According to algorithm A in \cref{alg:2}, we allocate the $s$-th request to the channel $i^{*}$ where 

\begin{equation}
    \begin{aligned}
    i^{*}=&\arg\min_{i\in\mathcal{I}}\sum_{k\in\mathcal{K}}\phi_k^{s-1}(-\frac{\ln(1-\varepsilon)}{\bar{a}_k})\exp(-\frac{\ln(1-\varepsilon)}{\bar{a}_k}(a_{isk} - \frac{U_{k}}{T}))\\
        &+ \sum_{k\in\mathcal{K}}\varphi_k^{s-1}(-\frac{\ln(1-\varepsilon)}{\bar{a}_k})\exp(-\frac{\ln(1-\varepsilon)}{\bar{a}_k}(\frac{L_k}{T} - a_{isk})) \\
        &+ \psi^{s-1}(-\frac{\ln(1-\varepsilon)}{\bar{w}})\exp(-\frac{\ln(1-\varepsilon)}{\bar{w}}(\frac{(1 - 2\varepsilon)W_{\tau}}{T} - w_{is}))
    \end{aligned}
\end{equation}
which means that 
\begin{equation}\label{equ::22extra1}
    \begin{aligned}
        \mathcal{F}(A^s\widetilde{P}^{T-s}) \le& \Big(\E\Big[\sum_{k\in\mathcal{K}}\phi_k^{s-1}(\frac{-\ln(1-\varepsilon)}{\bar{a}_k})\underbrace{\exp((-\frac{\ln(1-\varepsilon)}{\bar{a}_k})(X_{sk}^{\widetilde{P}} - \frac{U_{k}}{T}))}_{\textcircled{1}}\\
        &+ \sum_{k\in\mathcal{K}}\varphi_k^{s-1}(-\frac{\ln(1-\varepsilon)}{\bar{a}_k})\underbrace{\exp((-\frac{\ln(1-\varepsilon)}{\bar{a}_k})(\frac{L_k}{T} - X_{sk}^{\widetilde{P}}))}_{\textcircled{2}} \\
        &+ \psi^{s-1}(\frac{-\ln(1-\varepsilon)}{\bar{w}})\underbrace{\exp((\frac{-\ln(1-\varepsilon)}{\bar{w}})(\frac{(1 - 2\varepsilon)W_{\tau}}{T} - Y_s^{\widetilde{P}}))}_{\textcircled{3}}\Big]\Big)\exp\big(-\frac{T-s}{T}\frac{\varepsilon^2}{2(1-\frac{2}{3}\varepsilon)\gamma}\big).
    \end{aligned}
\end{equation}
For the term \textcircled{1}, we have 
\begin{equation}\label{equ:22similarAnalysis}
    \begin{aligned}
        \textcircled{1} & = \E\left[\exp\bigg(-\frac{\ln(1-\varepsilon)}{\bar{a}_k}\Big((X_{sk}^{\widetilde{P}} - \E[X_{sk}^{\widetilde{P}}]) + (\E[X_{sk}^{\widetilde{P}}] - \frac{U_{k}}{T})\Big)\bigg)\right] \\
        &\le \E\left[ \exp\left(\frac{\sigma^2}{\bar{a}_k^2}\Big(e^{-\ln(1-\varepsilon)} - 1 + \ln(1-\varepsilon)\Big) + \frac{\varepsilon U_k}{T\bar{a}_k}\ln(1-\varepsilon)\right)\right]\\
        & \le \E\left[ \exp\bigg(\frac{(1-\varepsilon)U_k}{T\bar{a}_k}\Big(\frac{1}{1-\varepsilon} - 1 + \ln(1-\varepsilon) + \frac{\varepsilon}{1-\varepsilon}\ln(1-\varepsilon)\Big)\bigg)\right]\\
        &= \E\left[ \exp\bigg(-\frac{(1-\varepsilon)U_k}{T\bar{a}_k}\Big((1 + \eta)\ln(1 + \eta) - \eta\Big)\bigg)\right]\\
        &\le\exp\Big(-\frac{1}{T}\frac{\varepsilon^2}{2(1-\frac{2}{3}\varepsilon)\gamma}\Big)
    \end{aligned}
\end{equation}
where the first inequality follows from \cref{lemma:2} and $\E[X_{sk}^{\widetilde{P}}]\le (1-\varepsilon)U_k/T$, the second from $\sigma^2 = Var(X_{sk}^{\widetilde{P}})\le\frac{(1-\varepsilon)\bar{a}_kU_k}{T}$. Next setting $\eta = \frac{\varepsilon}{1-\varepsilon}$, the last inequality follows from $(1+\eta)\ln(1+\eta)-\eta\ge\frac{\eta^2}{2+\frac{2}{3}\eta}$ and the definition   of $\gamma$.
For the term \textcircled{2}, we have
\begin{equation}\label{equ:23similarAnalysis}
    \begin{aligned}
        \textcircled{2} & = \E\left[\exp\bigg(-\frac{\ln(1-\varepsilon)}{\bar{a}_k}\Big((\E[X_{sk}^{\widetilde{P}}]- X_{sk}^{\widetilde{P}}) + (\frac{L_{k}}{T} -\E[X_{sk}^{\widetilde{P}}] )\Big)\bigg)\right] \\
        &\le \E\left[ \exp\left(\frac{\sigma^2}{\bar{a}_k^2}\Big(e^{-\ln(1-\varepsilon)} - 1 + \ln(1-\varepsilon)\Big) + \frac{\varepsilon (T\bar{a}_k - L_k)}{T\bar{a}_k}\ln(1-\varepsilon)\right)\right]\\
        & \le \E\left[ \exp\bigg(\frac{(1-\varepsilon)(T\bar{a}_k - L_k)}{T\bar{a}_k}\Big(\frac{1}{1-\varepsilon} - 1 + \ln(1-\varepsilon) + \frac{\varepsilon}{1-\varepsilon}\ln(1-\varepsilon)\Big)\bigg)\right]\\
        &= \E\left[ \exp\bigg(-\frac{(1-\varepsilon)(T\bar{a}_k - L_k)}{T\bar{a}_k}\Big((1 + \eta)\ln(1 + \eta) - \eta\Big)\bigg)\right]\\
        &\le\exp\Big(-\frac{1}{T}\frac{\varepsilon^2}{2(1-\frac{2}{3}\varepsilon)\gamma}\Big)
    \end{aligned}
\end{equation}
where the first inequality follows from \cref{lemma:2} and $\E[X_{sk}^{\widetilde{P}}]\le \frac{(1-\varepsilon)L_k + \varepsilon T\bar{a}_k}{T}$, the second from $\sigma^2 = Var(X_{sk}^{\widetilde{P}})\le\bar{a}_{k}\frac{(1-\varepsilon)(T\bar{a}_{k}-L_{k})}{T}$. Next setting $\eta = \frac{\varepsilon}{1-\varepsilon}$, the last inequality follows from $(1+\eta)\ln(1+\eta)-\eta\ge\frac{\eta^2}{2+\frac{2}{3}\eta}$ and the definition   of $\gamma$.

For the term \textcircled{3}, we have
\begin{equation}\label{equ:24similarAnalysis}
    \begin{aligned}
        \textcircled{3} & = \E\left[\exp\bigg(-\frac{\ln(1-\varepsilon)}{\bar{w}}\Big((\E[Y_s^{\widetilde{P}}]- Y_s^{\widetilde{P}}) + (\frac{(1-2\varepsilon)W_{\tau}}{T} -\E[Y_{s}^{\widetilde{P}}] )\Big)\bigg)\right] \\
        &\le \E\left[ \exp\left(\frac{\sigma_1^2}{\bar{w}^2}\Big(e^{-\ln(1-\varepsilon)} - 1 + \ln(1-\varepsilon)\Big) + \frac{\varepsilon W_{\tau}}{T\bar{w}}\ln(1-\varepsilon)\right)\right]\\
        & \le \E\left[ \exp\bigg(\frac{(1-\varepsilon)W_{\tau}}{T\bar{w}}\Big(\frac{1}{1-\varepsilon} - 1 + \ln(1-\varepsilon) + \frac{\varepsilon}{1-\varepsilon}\ln(1-\varepsilon)\Big)\bigg)\right]\\
        &= \E\left[ \exp\bigg(-\frac{(1-\varepsilon)W_{\tau}}{T\bar{w}}\Big((1 + \eta)\ln(1 + \eta) - \eta\Big)\bigg)\right]\\
        &\le\exp\Big(-\frac{1}{T}\frac{\varepsilon^2}{2(1-\frac{2}{3}\varepsilon)\gamma}\Big)
    \end{aligned}
\end{equation}
where the first inequality follows from \cref{lemma:2} and $\E[Y_{s}^{\widetilde{P}}]= \frac{(1-\varepsilon)W_{\tau}}{T}$, the second from $\sigma_1^2 = Var(Y_{s}^{\widetilde{P}})\le \frac{(1-\varepsilon)\bar{w}W_{\tau}}{T}$. Next setting $\eta = \frac{\varepsilon}{1-\varepsilon}$, the last inequality follows from $(1+\eta)\ln(1+\eta)-\eta\ge\frac{\eta^2}{2+\frac{2}{3}\eta}$ and the definition   of $\gamma$.
According to the inequality \eqref{equ:22similarAnalysis}-\eqref{equ:24similarAnalysis}, we can show that
\begin{equation}\label{equ::22extra2}
    \begin{aligned}
        \mathcal{F}(A^s\widetilde{P}^{T-s}) \le& \Big(\E\Big[\sum_{k\in\mathcal{K}}\phi_k^{s-1}(\frac{-\ln(1-\varepsilon)}{\bar{a}_k}) + \sum_{k\in\mathcal{K}}\varphi_k^{s-1}(-\frac{\ln(1-\varepsilon)}{\bar{a}_k})\\
        &+ \psi^{s-1}(\frac{-\ln(1-\varepsilon)}{\bar{w}})\Big]\Big)\exp(-\frac{T-s+1}{T}\frac{\varepsilon^2}{2(1-\frac{2}{3}\varepsilon)\gamma})\\
        =& \mathcal{F}(A^{s-1}\widetilde{P}^{T-s+1})
    \end{aligned}
\end{equation}
which completes the proof.

\end{proof}

\section{Proof of \cref{thm:4}}\label{appendix:thm4}
\subsection{Concentration of $Z^r$}\label{sec:concentration_z}
In the first step, we study the relationship between $Z^r$ and $W_{\varepsilon}$. 

\begin{lemma}\label{lemma:5ub}
Under \cref{assumption:1}-\ref{assumption:3}, if $\tau_1+\varepsilon\le\xi^*$ and $\gamma_{1}=\max\big(\frac{\bar{a}_{k}}{U_{k}}, \frac{\bar{a}_{k}}{(1-\varepsilon)T\bar{a}_{k}-L_{k}}, \frac{\bar{w}}{W_{\varepsilon+\tau_1}}\big)=O\big(\frac{\varepsilon^{2}}{\ln(K/\varepsilon)}\big)$, for given measure of feasibility $\xi^*$ , we have 
\begin{displaymath}
  W^{r}\le  \frac{t_{r}W_{\varepsilon}}{T}\bigg(1+\big(2+\frac{1}{\xi^*-\varepsilon}\big)\varepsilon_{x,r}\bigg)
\end{displaymath}
with probability at least $1-\delta$, where the predefined parameter $\varepsilon>0$, $\tau_1=\frac{\sqrt{\varepsilon}}{1-\sqrt{\varepsilon}}$, $\delta=\frac{\varepsilon}{3l}$, $l=\log_{2}(\frac{1}{\varepsilon})$ and $\varepsilon_{x,r}=\sqrt{\frac{4T\gamma_1 \ln(\frac{2K + 1}{\delta})}{t_r}}$.
\end{lemma}

\begin{proof}

We consider the definition of $W^r$:
	\begin{equation}\label{equ:a1}
	\begin{aligned}
	 W^{r}=&\max_{x}\ \sum_{i\in\mathcal{I},j\in\mathcal{S}_r} w_{ij}x_{ij}\\
					s.t.\ &\frac{t_{r}}{T}(L_{k}+\varepsilon T\bar{a}_{k})\le\sum_{i\in\mathcal{I},j\in\mathcal{S}_r} a_{ijk}x_{ij}\le\frac{t_{r}}{T}U_{k}, \forall k\in\mathcal{K}\\
					&\sum_{i\in\mathcal{I}} x_{ij}\le 1, \forall j\in\mathcal{S}_r\\
					&x_{ij}\ge 0, \forall i\in\mathcal{I}, j\in\mathcal{S}_r
	\end{aligned}
	\end{equation}
	where we use $\mathcal{S}_r$ to denote the request set in stage $r$. The dual of LP~\eqref{equ:a1} is 
	
	\begin{equation}\label{equ:26}
		\begin{aligned}
		W^{r}=\min_{\alpha,\beta,\rho}\ &\sum_{k\in\mathcal{K}} \alpha_{k}\frac{t_{r}}{T}U_{k}-\sum_{k\in\mathcal{K}}\beta_{k}\frac{t_{r}}{T}(L_{k}+\varepsilon T\bar{a}_{k})+\sum_{j\in \mathcal{S}_r}\rho_j \\
			s.t.\ &\sum_{k\in\mathcal{K}}(\alpha_{k}-\beta_{k})a_{ijk}-w_{ij}+\rho_{j}\ge0\ \forall i\in\mathcal{I},j\in\mathcal{S}_r\\
			&\alpha_{k},\beta_{k},\rho_{j}\ge 0,k\in\mathcal{K}, j\in\mathcal{S}_r
		\end{aligned}
	\end{equation}
	Comparing to the dual of LP~\eqref{equ:a1} with the dual of problem $E(\varepsilon)$, which is
	\begin{equation}\label{equ:25}
		\begin{aligned}
		\min_{\alpha,\beta,\rho}\ &\sum_{k\in\mathcal{K}} \alpha_{k}U_{k}-\sum_{k\in\mathcal{K}}\beta_{k}(L_{k}+\varepsilon T\bar{a}_{k})+\sum_{j\in\mathcal{J}}T p_{j}\rho_j \\
			s.t.\ &\sum_{k\in\mathcal{K}}(\alpha_{k}-\beta_{k})a_{ijk}-w_{ij}+\rho_{j}\ge0\ \forall i\in\mathcal{I},j\in\mathcal{J}\\
			&\alpha_{k},\beta_{k},\rho_{j}\ge 0,k\in\mathcal{K}, j\in\mathcal{J}
		\end{aligned}
	\end{equation}
	we can observe that the constraints of LP~\eqref{equ:26} is a subset of those of LP~\eqref{equ:25}. We denote the primal and dual optimal solution of $E(\varepsilon)$ as $\{x_{ij}^{*}\}$  and $\{\alpha_k^{*},\beta_{k}^{*},\rho_{k}^{*}\}$. So  $\{\alpha_k^{*},\beta_{k}^{*},\rho_{k}^{*}\}$ is feasible for LP~\eqref{equ:26}.
	
	
	
	Hence,
	\begin{equation}\label{equ:28}
			\begin{aligned}
			W^{r}&\le\sum_{k\in\mathcal{K}} \alpha_{k}^{*}\frac{t_r}{T}U_{k}-\sum_{k\in\mathcal{K}}\beta_{k}^{*}\frac{t_r}{T}(L_{k}+\varepsilon T\bar{a}_{k})+\sum_{j\in \mathcal{S}_r}\rho^{*}_j\\
			&=\underbrace{\sum_{k\in\mathcal{K}} \alpha_{k}^{*}(\frac{t_r}{n}U_{k}-\sum_{j\in \mathcal{S}_r,i\in\mathcal{I}}a_{ijk}x_{ij}^{*})}_{\textcircled{1}}+\underbrace{\sum_{k\in\mathcal{K}}\beta_{k}^{*}(\sum_{j\in \mathcal{S}_r,i\in\mathcal{I}}a_{ijk}x_{ij}^{*}-\frac{t_r}{T }(L_{k}+\varepsilon T\bar{a}_{k}))}_{\textcircled{2}}\\&+\underbrace{\sum_{j\in \mathcal{S}_r}(\rho^{*}_{j}+\sum_{i\in\mathcal{I},k\in\mathcal{K}}(\alpha^{*}_{k}-\beta^{*}_{k})a_{ijk}x_{ij}^{*})}_{\textcircled{3}}  
			\end{aligned}
	\end{equation}
	We have divided the equation~\eqref{equ:28} into three parts. Next, we will derive the relationship between $W^{r}$ and $W_{\varepsilon}$ by controlling these three parts.  To facilitate the analysis, we first present the KKT conditions\citep{boyd2004convex} for the problem $E(\varepsilon) $ as follows
		\begin{equation}\label{equ:27}
			\begin{aligned}
			\sum_{k\in\mathcal{K}}&(\alpha^{*}_{k}-\beta^{*}_{k})a_{ijk}x_{ij}^{*}-w_{ij}x_{ij}^{*}+\rho_{t}^{*}x_{ij}^{*}=0, \forall i\in \mathcal{I}, j\in\mathcal{S}_r\\
			\rho_j^{*}&(\sum_{i\in\mathcal{I}} x_{ij}^{*}-1)=0, \forall j\in\mathcal{S}_r
		 \\
			\alpha_{k}^{*}&(\sum_{ij}T p_{j}a_{ijk}x_{ij}^{*}-U_k)=0, \forall k \in \mathcal{K}\\
			\beta_{k}^{*}&(L_{k}+\varepsilon T\bar{a}_{k}-\sum_{ij}T p_{j}a_{ijk}x_{ij}^{*})=0, \forall k \in \mathcal{K}.
			 \end{aligned}
	\end{equation}

	For part $\textcircled{1}$, according to the KKT conditions, if $\sum_{i\in\mathcal{I},j\in\mathcal{J}}T p_{j}a_{ijk}x_{ij}^{*}<U_k$, then $\alpha_{k}^{*}=0$. Thus we only consider the resource $k$ making $\sum_{i\in\mathcal{I},j\in\mathcal{J}}T p_{j}a_{ijk}x_{ij}^{*}=U_k$. According to the \cref{assumption:1}, we have $\mathbb{E}(\sum_{j\in\mathcal{S}_r,i\in\mathcal{I}}a_{ijk}x_{ij}^{*})=\frac{t_r}{T}U_{k}$ $\forall j\in\mathcal{S}_r$. Thus, we can show that 
	\begin{equation}\label{equ:29}
			\begin{aligned}
					&P\left(\sum_{j\in\mathcal{S}_r,i\in\mathcal{I}}a_{ijk}x_{ij}^{*}\le(1-\varepsilon_{x,r})\frac{t_{r}}{T}U_{k}\right)\\
					&= P\left(\sum_{j\in\mathcal{S}_r,i\in\mathcal{I}}a_{ijk}x_{ij}^{*}-\mathbb{E}[\sum_{j\in\mathcal{S}_r,i\in\mathcal{I}}a_{ijk}x_{ij}^{*}]\le-\varepsilon_{x,r}\frac{t_r}{T}U_k\right)\\
					&\le \exp\left(-\frac{t_r U_k^2\varepsilon_{x,r}^2/T^2}{2Var(\sum_{j\in\mathcal{S}_r,i\in\mathcal{I}}a_{ijk}x_{ij}^*)/t_{r} + \frac{2}{3}\bar{a}_kU_k\varepsilon_{x,r}/T}\right)\\
					&\le \exp\left(-\frac{\frac{t_{r}}{T}\varepsilon_{x,r}^2}{2(1+\frac{\varepsilon_{x,r}}{3})\frac{\bar{a}_{k}}{U_k}}\right)\\
					&\le  \exp\left(-\frac{\frac{t_{r}}{T}\varepsilon_{x,r}^2}{2(1+\frac{\varepsilon_{x,r}}{3})\gamma_1}\right)\\
					&\le \frac{\delta}{2K + 1}
			\end{aligned}
	\end{equation} where the first inequality follows from the Bernstein inequality in \cref{lemma:2}, the second inequality from $Var(\sum_{j\in\mathcal{S}_r,i\in\mathcal{I}}a_{ijk}x_{ij}^*)/t_{r}\le \bar{a}_k\E[\sum_{j\in\mathcal{S}_r,i\in\mathcal{I}}a_{ijk}x_{ij^*}]/t_{r}=\bar{a}_{k}\frac{U_{k}}{T}$, the third inequality from the definition of $\gamma_1$ and the last from the definition of $\varepsilon_{x,r}$.
	
	For part $\textcircled{2}$, we only consider the  $k$ making $\sum_{i\in\mathcal{I},j\in\mathcal{J}}T p_{j}a_{ijk}x_{ij}^{*}=L_k+\varepsilon T\bar{a}_{k}$, which means $\E\Big[\sum_{j\in\mathcal{S}_r,i\in\mathcal{I}}(\bar{a}_{k}-a_{ijk}x_{ij}^{*})\Big] = \frac{t_r}{T}\Big((1-\varepsilon)T\bar{a}_k-L_k\Big)$. Using Bernstein inequality, we have 
	\begin{equation}\label{equ:30}
			\begin{aligned}
					&P\left(\sum_{j\in\mathcal{S}_r,i\in\mathcal{I}}(\bar{a}_{k}-a_{ijk}x_{ij}^{*})\le(1-\varepsilon_{x,r})\frac{t_r}{T}((1-\varepsilon)T\bar{a}_{k}-L_{k})\right)\\
					& = P\left(\sum_{j\in\mathcal{S}_r,i\in\mathcal{I}}(\bar{a}_{k}-a_{ijk}x_{ij}^{*})-\mathbb{E}\Big[\sum_{j\in\mathcal{S}_r,i\in\mathcal{I}}(\bar{a}_{k}-a_{ijk}x_{ij}^{*})\Big]\le -\varepsilon_{x,r}\frac{t_r}{T}((1-\varepsilon)T\bar{a}_{k}-L_{k})\right) \\
					&\le \exp\left(-\frac{t_r \Big((1-\varepsilon)T\bar{a}_{k}-L_{k}\Big)^2\varepsilon_{x,r}^2/T^2}{2Var\Big(\sum_{j\in\mathcal{S}_r,i\in\mathcal{I}}(\bar{a}_{k}-a_{ijk}x_{ij}^{*})\Big)/t_{r} + \frac{2}{3}\bar{a}_k\Big((1-\varepsilon)T\bar{a}_{k}-L_{k}\Big)\varepsilon_{x,r}/T}\right)\\
					&\le \exp\left(-\frac{\frac{t_{r}}{T}\varepsilon_{x,r}^2}{2(1+\frac{\varepsilon_{x,r}}{3})\frac{\bar{a}_{k}}{(1-\varepsilon)T\bar{a}_{k}-L_{k}}}\right)\\
					&\le  \exp\left(-\frac{\frac{t_{r}}{T}\varepsilon_{x,r}^2}{2(1+\frac{\varepsilon_{x,r}}{3})\gamma_1}\right)\\
					&\le \frac{\delta}{2K + 1}
			\end{aligned}
	\end{equation}
	where the second inequality from $Var\Big(\sum_{j\in\mathcal{S}_r,i\in\mathcal{I}}(\bar{a}_{k}-a_{ijk}x_{ij}^{*})\Big)/t_r \le \bar{a}_k\E[\sum_{j\in\mathcal{S}_r,i\in\mathcal{I}}(\bar{a}_{k}-a_{ijk}x_{ij}^{*})]/t_{r}=\bar{a}_{k}\frac{(1-\varepsilon)T\bar{a}_{k}-L_{k}}{T}$, the third inequality from the definition of $\gamma_1$ and the last from the definition of $\varepsilon_{x,r}$.
	
	For the last part $\textcircled{3}$, from KKT conditions, it's easy to verify that $\sum_{i\in\mathcal{I},k\in\mathcal{K}}(\alpha^{*}_{k}-\beta^{*}_{k})a_{ijk}x_{ij}^{*}-\sum_{i\in\mathcal{I}}w_{ij}x_{ij}^{*}+\rho_{j}^{*}\sum_{i\in\mathcal{I}}x_{ij}^{*}=\sum_{i\in\mathcal{I},k\in\mathcal{K}}(\alpha^{*}_{k}-\beta^{*}_{k})a_{ijk}x_{ij}^{*}-\sum_{i\in\mathcal{I}}w_{ij}x_{ij}^{*}+\rho_{j}^{*}=0$, so $\sum_{i\in\mathcal{I},k\in\mathcal{K}}(\alpha^{*}_{k}-\beta^{*}_{k})a_{ijk}x_{ij}^{*}+\rho_{j}^{*}=\sum_{i\in\mathcal{I}}w_{ij}x_{ij}^{*}\in[0,\bar{w}]$. Moreover, $\mathbb{E}[\rho^{*}_{j}+\sum_{i\in\mathcal{I},k\in\mathcal{K}}(\alpha^{*}_{k}-\beta^{*}_{k})a_{ijk}x_{ij}^{*}]=\frac{\sum_{i\in\mathcal{I},j\in\mathcal{J}}Tp_{j}w_{ij}x^{*}_{ij}}{T}=\frac{W_{\varepsilon}}{T}$ $\forall j\in\mathcal{S}_r$. Therefore, following the similar analysis as equation \eqref{equ:29}, we have 
	\begin{equation}\label{equ:31}
			\begin{aligned}
			&P\left(\sum_{j\in \mathcal{S}_r}\Big(\rho^{*}_{j}+\sum_{i\in\mathcal{I},k\in\mathcal{K}}(\alpha^{*}_{k}-\beta^{*}_{k})a_{ijk}x_{ij}^{*}\Big)\ge\frac{t_r}{T}W_{\varepsilon}(1+\varepsilon_{x,r})\right)\\
			&=P(\sum_{j\in \mathcal{S}_r, i \in \mathcal{I}} w_{ij}x_{ij}^* - \E[\sum_{j\in \mathcal{S}_r, i \in \mathcal{I}} w_{ij}x_{ij}^*] \ge \varepsilon_{x,r}\frac{t_r}{T}W_{\varepsilon})\\
			&\le \exp\left(-\frac{t_r W_{\varepsilon}^2\varepsilon_{x,r}^2/T^2}{2Var(\sum_{j\in\mathcal{S}_r,i\in\mathcal{I}}w_{ij}x_{ij}^*)/t_{r} + \frac{2}{3}\bar{a}_kW_{\varepsilon}\varepsilon_{x,r}/T}\right)\\
			&\le \exp(-\frac{\frac{t_{r}}{T}\varepsilon_{x,r}^2}{2(1+\frac{\varepsilon_{x,r}}{3})\frac{\bar{w}}{W_{\varepsilon}}})\\
			&\le  \exp\left(-\frac{\frac{t_{r}}{T}\varepsilon_{x,r}^2}{2(1+\frac{\varepsilon_{x,r}}{3})\gamma_1}\right)\\
			&\le \frac{\delta}{2K + 1}
			\end{aligned}
	\end{equation}
	where the first inequality follows from the Bernstein inequality in \cref{lemma:2}, the second inequality from $Var(\sum_{j\in\mathcal{S}_r,i\in\mathcal{I}}w_{ij}x_{ij}^*)/t_{r}\le \bar{w}\E[\sum_{j\in\mathcal{S}_r,i\in\mathcal{I}}w_{ij}x_{ij}^*]/t_{r}=\bar{w}\frac{W_{\varepsilon}}{T}$, the third inequality from the definition of $\gamma_1$ and the last from the definition of $\varepsilon_{x,r}$.
	
	Based on inequality~\eqref{equ:27}-\eqref{equ:30}, we have shown that the following inequalities holds with probability at least $1-\delta$ 
	\begin{equation}\label{equ:32}
			\begin{aligned}
			&\sum_{j\in\mathcal{S}_r,i\in\mathcal{I}}a_{ijk}x_{ij}^{*}\ge(1-\varepsilon_{x,r})\frac{t_{r}}{T}U_{k},\forall k \in \mathcal{K}\\
			&\sum_{j\in\mathcal{S}_r,i\in\mathcal{I}}(\bar{a}_{k}-a_{ijk}x_{ij}^{*})\ge(1-\varepsilon_{x,r})\frac{t_r}{T}((1-\varepsilon)T\bar{a}_{k}-L_{k}),\forall k \in \mathcal{K}\\
			&\sum_{j\in \mathcal{S}_r}(\rho^{*}_{j}+\sum_{i\in\mathcal{I},k\in\mathcal{K}}(\alpha^{*}_{k}-\beta^{*}_{k})a_{ijk}x_{ij}^{*})\le\frac{t_r}{T}W_{\varepsilon}(1+\varepsilon_{x,r}),\forall k \in \mathcal{K}.
			\end{aligned}
	\end{equation} 

	Therefore, with probability at least $1-\delta$, we have
	\begin{equation}\label{equ:33}
			\begin{aligned}
			&\textcircled{1}+\textcircled{2}+\textcircled{3}\\
			&=\sum_{k\in\mathcal{K}} \alpha_{k}^{*}\big(\frac{t_r}{T}U_{k}-\sum_{j\in \mathcal{S}_r,i\in\mathcal{I}}a_{ijk}x_{ij}^{*}\big)+\sum_{k\in\mathcal{K}}\beta_{k}^{*}\big(\frac{t_r}{T}\big((1-\varepsilon)T\bar{a}_{k}-L_{k}\big)-\sum_{j\in \mathcal{S}_r,i\in\mathcal{I}}(\bar{a}_{k}-a_{ijk}x_{ij}^{*})\big)+\textcircled{3}\\
			&\le\varepsilon_{x,r}\sum_{k\in\mathcal{K}}\bigg(\alpha^{*}_{k}\frac{t_r}{T}U_{k}+\beta^{*}_{k}\Big(\frac{t_r}{T}\big((1-\varepsilon)T\bar{a}_{k}-L_{k}\big)\Big)\bigg)+\textcircled{3}\\
			&\le\varepsilon_{x,r}\sum_{k\in\mathcal{K}}\bigg(\alpha^{*}_{k}\frac{t_r}{T}U_{k}+\beta^{*}_{k}\Big(\frac{t_r}{T}\big((1-\varepsilon)T\bar{a}_{k}-L_{k}\big)\Big)\bigg)+\frac{t_r}{T}W_{\varepsilon}(1+\varepsilon_{x,r})\\			&\le\varepsilon_{x,r}\big(\frac{t_r}{T}W_{\varepsilon}+\sum_{k\in\mathcal{K}}\beta^{*}_{k}t_{r}\bar{a}_{k}\big)+\frac{t_r}{T}W_{\varepsilon}(1+\varepsilon_{x,r})\\
			&\le \big(1+(2+\frac{1}{\xi^*-\varepsilon})\varepsilon_{x,r}\big)\frac{t_{r}}{T}W_{\varepsilon}
			\end{aligned}
	\end{equation} where the first inequality from $\sum_{j\in\mathcal{S}_r,i\in\mathcal{I}}a_{ijk}x_{ij}^{*}\ge(1-\varepsilon_{x,r})\frac{t_{r}}{T}U_{k}$ and $\sum_{j\in\mathcal{S}_r,i\in\mathcal{I}}(\bar{a}_{k}-a_{ijk}x_{ij}^{*})\ge(1-\varepsilon_{x,r})\frac{t_r}{T}((1-\varepsilon)T\bar{a}_{k}-L_{k})$; the second inequality from $\sum_{j\in \mathcal{S}_r}(\rho^{*}_{j}+\sum_{i\in\mathcal{I},k\in\mathcal{K}}(\alpha^{*}_{k}-\beta^{*}_{k})a_{ijk}x_{ij}^{*})\le\frac{t_r}{T}W_{\varepsilon}(1+\varepsilon_{x,r})$; the third inequality from $W_{\varepsilon}=\sum_{k\in\mathcal{K}} \alpha^{*}_{k}U_{k}-\sum_{k\in\mathcal{K}}\beta^{*}_{k}(L_{k}+\varepsilon T\bar{a}_{k})+\sum_{j\in\mathcal{J}}T p_{j}\rho^{*}_j\ge\sum_{k\in\mathcal{K}} \alpha_{k}^{*}U_{k}-\sum_{k\in\mathcal{K}}\beta_{k}^{*}(\varepsilon T\bar{a}_{k}+L_{k})$; the final inequality from $\sum_{k\in\mathcal{K}}\beta_{k}^{*}T\bar{a}_{k}\le\frac{1}{\xi^*-\varepsilon}W_{\varepsilon}$, which can be shown if we follow the proof of \cref{lemma:3} in \cref{appendix::factor_proof} and regard the problem $E(\varepsilon)$ as an LP with $\xi^* - \varepsilon$ measure of feasibility. 
	\end{proof}
	
	\begin{theorem}\label{thm:5}
Under \cref{assumption:1}-\ref{assumption:3}, if $\tau_1+\varepsilon\le\xi^*$ and $\gamma_{1}=\max\big(\frac{\bar{a}_{k}}{U_{k}}, \frac{\bar{a}_{k}}{(1-\varepsilon)T\bar{a}_{k}-L_{k}}, \frac{\bar{w}}{W_{\varepsilon+\tau_1}}\big)=O\big(\frac{\varepsilon^{2}}{\ln(K/\varepsilon)}\big)$, with probability $1-2\delta$, we have 
\begin{displaymath}
 \frac{t_{r}W_{\varepsilon}}{T}\bigg(1-\big(2+\frac{1}{\xi^*-\varepsilon}\big)\varepsilon_{x,r}\bigg)\le W^{r}\le  \frac{t_{r}W_{\varepsilon}}{T}\bigg(1+\big(2+\frac{1}{\xi^*-\varepsilon}\big)\varepsilon_{x,r}\bigg)
\end{displaymath}
\end{theorem} 
where the predefined parameter $\varepsilon>0$, $\tau_1=\frac{\sqrt{\varepsilon}}{1-\sqrt{\varepsilon}}$, $\delta=\frac{\varepsilon}{3l}$, $l=\log_{2}(\frac{1}{\varepsilon})$ and $\varepsilon_{x,r}=\sqrt{\frac{4T\gamma_1 \ln(\frac{2K + 1}{\delta})}{t_r}}$.
	\begin{proof} 
	\textbf{RHS:} we have proven the right hand side in \cref{lemma:5ub}.\\
	\textbf{LHS:} For every request in stage $r$, we consider to imitate \cref{alg:framework} to design an algorithm $\widetilde{P}_{1}$. In Algorithm $\widetilde{P}_{1}$, we first solve the LP problem $E(\varepsilon+\frac{\varepsilon_{x,r}}{1-\varepsilon_{x,r}})$ to get the LP solution $x_{ij}^{1*}$ and then assign request $j\in\mathcal{J}$ to channel $i\in\mathcal{I}$ with probability $(1-\varepsilon_{x,r})x_{ij}^{1*}$. 
	Following the similar analysis in the proof of \cref{thm:2}, we can prove that , with probability $1-\delta$, 
	\begin{equation}\label{equ:34}
		\begin{aligned}
		&P\left(\sum_{j=1}^{t_{r}}X^{\widetilde{P}_1}_{jk}\ge\frac{t_{r}}{T}U_k\right)\le \exp\left(-\frac{t_{r}\varepsilon_{x,r}^2}{2(1-\frac{2}{3}\varepsilon_{x,r})T\frac{\bar{a}_{k}}{U_k}}\right)\\
			&P\left(\sum_{j=1}^{t_{r}}Y^{\widetilde{P}_1}_{j}\le (1-(2+\frac{1}{\xi^*-\varepsilon})\varepsilon_{x,r})\frac{t_{r}}{T}W_{\varepsilon}\right)\le \exp\left(-\frac{t_{r}\varepsilon_{x,r}^2}{2(1-\frac{2}{3}\varepsilon_{x,r})T\frac{\bar{w}}{W_{\tau_{1}+\varepsilon}}}\right)\\
		&P\left(\sum_{j=1}^{t_{r}}X^{\widetilde{P}_1}_{jk}\le \frac{t_{r}}{T}L_k\right)\le \exp\left(-\frac{t_{r}\varepsilon_{x,r}^2}{2(1-\frac{2}{3}\varepsilon_{x,r})T\frac{\bar{a}_{k}}{(1-\varepsilon-\tau_{1})T\bar{a}_{k}-L_{k}}}\right)
		\end{aligned} 
	\end{equation} where the second inequality from the truth that the problem $E(\varepsilon)$ is satisfied with the strong feasible condition with the measure parameter $\xi^*-\varepsilon$ and $W_{\frac{\varepsilon_{x,r}}{1-\varepsilon_{x,r}}+\varepsilon}\ge W_{\tau_{1}+\varepsilon}$. Therefore, we have that
	\begin{equation}\label{equ:a2}
		\begin{aligned}
		&P\left(\sum_{j=1}^{t_{r}}Y^{\widetilde{P}_1}_{j}\le \Big(1-(2+\frac{1}{\xi^*-\varepsilon})\varepsilon_{x,r}\Big)\frac{t_{r}}{T}W_{\varepsilon}\right)+\sum_{k\in\mathcal{K}}P\left(\sum_{j=1}^{t_{r}}X^{\widetilde{P}_1}_{jk}\notin[\frac{t_{r}}{T}L_k,\frac{t_{r}}{T}U_k]\right)\\
			&\le (2K+1)\exp(-\frac{t_{r}\varepsilon_{x,r}^2}{4T\gamma_1})\\
			&\le\delta,
		\end{aligned} 
	\end{equation} 
	
	which means that $W^{r}\ge(1-(2+\frac{1}{\xi^*-\varepsilon})\varepsilon_{x,r})\frac{t_{r}}{T}W_{\tau}$, w.p. $1-\delta$. This completes the proof.
	\end{proof}

\subsection{Proof with $\xi^*$ but without the Knowledge of Distribution}\label{sec:proof_without}
\thmunknowndis*
Now, we prove that, at each stage $r$, Algorithm $A_{1}$ returns a solution whose cumulative revenue is at least $\frac{t_{r}Z^r}{T}(1-\varepsilon_{y,r})$. Meanwhile, the consumed amount of every resource $k$ is between $\frac{t_{r}}{T}(L_{k}+(\varepsilon-\frac{\varepsilon_{x,r}}{1+\varepsilon_{x,r}})T\bar{a}_{k})(1+\varepsilon_{x,r})$ and $\frac{t_{r}U_{k}}{T}(1+\varepsilon_{x,r})$ with probability at least $1-\delta$.

\noindent\textbf{First step}: We design a surrogate Algorithm $\widetilde{P}_{2}$ that allocates request $j$ to channel $i$ with probability $x(\varepsilon)^{*}_{ij}$.
\begin{lemma}\label{lemma:a1}
In the $r$-th stage, if $\gamma_1=O(\frac{\varepsilon^2}{\ln(\frac{K}{\varepsilon})})$ and $Z^r \leq W_\varepsilon$, the Algorithm $\widetilde{P}_{2}$ returns a solution satisfying the $\sum_{j=t_r+1}^{t_{r+1}}Y_{j}^{\widetilde{P}_2}\ge (1-\varepsilon_{y,r})\frac{t_r}{n}Z^r$ and $\sum_{j=t_r+1}^{t_{r+1}}X_j^{\widetilde{P}_2} \in\left[\frac{t_{r}}{T}\Big((1+\varepsilon_{x,r}) L_{k} + (\varepsilon(1+\varepsilon_{x,r}) - \varepsilon_{x,r})T\bar{a}_{k}\Big)\right.$, $\left. \frac{(1+\varepsilon_{x,r})t_r}{T}U_k\right]$ w.p. $1-\delta$, where $\delta = \frac{\varepsilon}{3l}$, $\varepsilon_{x,r}=\sqrt{\frac{4T\gamma_1 \ln(\frac{2K + 1}{\delta})}{t_r}}$ and $\varepsilon_{y,r}=\sqrt{\frac{4T\ln(\frac{2K+1}{\delta})\bar{w}}{Z_rt_r}}$.
\end{lemma}
\begin{proof}
Following the same technique in the proof of \cref{thm:2}, we use Bernstein inequality to bound the probability of upper bound violation as follows
\begin{align*}
	&P\left(\sum_{j=t_r+1}^{t_{r+1}}X^{\widetilde{P}_{2}}_{jk}\ge(1+\varepsilon_{x,r})\frac{t_r}{T}U_k\right)\\
	&\le \exp\left(-\frac{\frac{t_{r}}{T}\varepsilon_{x,r}^2}{2(1+\frac{\varepsilon_{x,r}}{3})\frac{\bar{a}_{k}}{U_k}}\right)\\
	&\le \exp\left(-\frac{\frac{t_{r}}{T}\varepsilon_{x,r}^2}{2(1+\frac{\varepsilon_{x,r}}{3})\gamma_{1}}\right)\\
	&\le \frac{\delta}{2K+1}.
\end{align*}

For the lower bound,
\begin{align*}
&P\left(\sum_{j=1+t_r}^{t_{r+1}}X^{\widetilde{P}_{2}}_{jk}\le\frac{t_{r}}{T}(L_{k}+(\varepsilon-\frac{\varepsilon_{x,r}}{1+\varepsilon_{x,r}})T\bar{a}_{k})(1+\varepsilon_{x,r})\right)\\
	&=P\left(\sum_{j=1+t_r}^{t_{r+1}}(\bar{a}_{k}-X^{\widetilde{P}_{2}}_{jk})\ge(1+\varepsilon_{x,r}) \frac{t_{r}}{T}((1-\varepsilon) T\bar{a}_{k}-L_{k})\right)\\
	&\le \exp\left(-\frac{\frac{t_{r}}{T}\varepsilon_{x,r}^2}{2(1+\frac{1}{3}\varepsilon_{x,r})\frac{\bar{a}_{k}}{(1-\varepsilon)T\bar{a}_{k}-L_{k}}} \right)\\
	&\le \exp\left(-\frac{\frac{t_{r}}{T}\varepsilon_{x,r}^2}{2(1+\frac{\varepsilon_{x,r}}{3})\gamma_{1}}\right)\\
	&\le \frac{\delta}{2K+1}.
\end{align*}
For the accumulative revenue in $r$-th stage,
\begin{align*}
	&P\left(\sum_{j=t_r+1}^{t_{r+1}}Y^{\widetilde{P}_{2}}_{j}\le (1-\varepsilon_{y,r})\frac{t_r}{T}Z^r\right)\\
	&=P\left(\sum_{j=t_r+1}^{t_{r+1}}\Big(\mathbb{E}[Y^{\widetilde{P}_{2}}_{j}]-Y_{j}^{\widetilde{P}_{2}}\Big)\ge \frac{t_r}{T}\Big(T \mathbb{E}[Y^{\widetilde{P}_{2}}_{j}]-(1-\varepsilon_{y,r})Z^r\Big)\right)\\
	&\le \exp\left(-\frac{\Big(T \mathbb{E}[Y^{\widetilde{P}_{2}}_j]-(1-\varepsilon_{y,r})Z^r\Big)^{2}t_r}{T\bigg(2T\sigma_{1}^2+\frac{2}{3}\bar{w}\Big(T\mathbb{E}[Y^{\widetilde{P}_{2}}_j]-(1-\varepsilon_{y,r})Z^r\Big)\bigg)}\right)\\
	&=\exp\left(-\frac{\frac{t_{r}}{T}\Big(T\mathbb{E}[Y^{\widetilde{P}_{2}}_{j}]-(1-\varepsilon_{y,r})Z^r\Big)}{2\frac{T\sigma_1^2}{T\mathbb{E}[Y^{\widetilde{P}_{2}}_j]-(1-2\varepsilon_{y,r})Z^r}+\frac{2}{3}\bar{w}}\right)\\
	&\le \exp\left(-\frac{t_r\Big(\varepsilon_{y,r}T\mathbb{E}[Y_j^{\widetilde{P}_{2}}]\Big)^2}{T\Big(2T\mathbb{E}[Y_j^{\widetilde{P}_{2}}]\bar{w}+\frac{2}{3}\varepsilon_{y,r}\bar{w}T\mathbb{E}[Y_j^{\widetilde{P}_{2}}]\Big)}\right)\\&\le
	\exp\left(-\frac{t_r\varepsilon^2_{y,r}}{2(1+\frac{\varepsilon_{y,r}}{3})T\frac{\bar{w}}{Z^r}}\right)\\&\le
	\frac{\delta}{2K+1}
\end{align*}
where the second inequality follows $\sigma_{1}^2=Var(Y^{\widetilde{P}_{2}}_{j})\le \bar{w}\E[Y_{j}^{\widetilde{P}_{2}}]$ and $T\mathbb{E}[Y^{\widetilde{P}_{2}}_j]-(1-\varepsilon_{y,r})Z^r\ge\varepsilon_{y,r}T \mathbb{E}[Y^{\widetilde{P}_{2}}_j]$, the third inequality follows $Z^r\le W_{\varepsilon}= T\mathbb{E}[Y_j^{\widetilde{P}_{2}}]$ according to the condition of the lemma, and the last inequality from $\varepsilon_{y,r}=\sqrt{\frac{4 T\ln(\frac{2K + 1}{\delta})\bar{w}}{Z^rt_r}}$.

Therefore, 
\begin{align*}
    &\sum_{k\in\mathcal{K}}P\left(\sum_{j=t_r+1}^{t_{r+1}}X_{jk}\notin\left[\frac{t_{r}}{T}\Big((1+\varepsilon_{x,r})L_{k}-(\varepsilon(1+\varepsilon_{x,r})-\varepsilon_{x,r})T\bar{a}_{k}\Big),\frac{(1+\varepsilon_{x,r})t_r}{T}U_r\right]\right) \\
    &+P\left(\sum_{j=t_r+1}^{t_{r+1}}Y^{\widetilde{P}_{2}}_{j}\le (1-\varepsilon_{y,r})\frac{t_r}{T}Z^r\right)\le \delta.
\end{align*}
\end{proof}

\noindent\textbf{Second Step:} Applying the same technique in the proof of \cref{thm:3} , we  derive a potential function to bound the failure probability of hybrid Algorithm $A_1^s\widetilde{P}_{2}^{t_r -s}$ for request in stage $r$.

We begin with the moment generating function for the event that the consumed resource $k\in\mathcal{K}$ is larger than $(1+\varepsilon_{x,r})\frac{t_{r}}{T}U_{k}$. It can be shown that 
\begin{equation}\label{equ:38}
	\begin{aligned}
		&P\left(\sum_{j=1+t_r}^{s+t_r}X_{jk}^{A_1}+\sum_{j=s+t_r}^{t_{r+1}}X^{\widetilde{P}_{2}}_{jk}\ge\frac{(1+\varepsilon_{x,r})t_r U_k}{T}\right)\\
		&\le \min_{t>0}\E\left[\exp(t(\sum_{j=1+t_r}^{s+t_r}X^{A_1}_{jk}+\sum_{j=s+t_r}^{t_{r+1}}X^{\widetilde{P}_{2}}_{jk}-\frac{(1+\varepsilon_{x,r})t_r U_k}{T}))\right]\\
		&\le \min_{t>0}\E\left[\exp(t(\sum_{j=1+t_r}^{s+t_r}X^{A_1}_{jk}-\frac{(1+\varepsilon_{x,r})s U_k}{T})+t(\sum_{j=s+t_r}^{t_{r+1}}X^{\widetilde{P}_{2}}_{jk}-\frac{(1+\varepsilon_{x,r})(t_r-s) U_k}{T}))\right]\\
		&\le \min_{t>0}\mathbb{E}\left[\phi_k^s(t)\exp(t(\sum_{j=s+t_r}^{t_{r+1}}(X^{A_1}_{jk}-\mathbb{E}[X^{\widetilde{P}_{2}}_{jk}]))+\frac{t_r-s}{T}t(T\mathbb{E}[X^{\widetilde{P}_{2}}_{jk}]-(1+\varepsilon_{x,r})U_k))\right] \\
		&\le \min_{t>0}\mathbb{E}\left[\phi_k^s(t)\exp((t_r-s)\frac{Var(X_{jk}^{\widetilde{P}_2})}{\bar{a}_k^2}(e^{t\bar{a}_k}-1-t\bar{a}_k)+\frac{-(t_r-s)t\varepsilon_{x,r} U_k}{T})\right]\\
		&\le \mathbb{E}\left[\phi_k^s\left(\frac{\ln(1+\varepsilon_{x,r})}{\bar{a}_k}\right)\exp(-\frac{(t_r-s)U_k}{T\bar{a}_k}((1+\eta)\ln(1+\eta)-\eta))\right]\\
		&\le \mathbb{E}\left[\phi_k^s\left(\frac{\ln(1+\varepsilon_{x,r})}{\bar{a}_k}\right)\exp(-\frac{t_r-s}{T}\frac{\varepsilon_{x,r}^2}{4\gamma_1})\right]
	\end{aligned}
\end{equation} 
where $\phi_k^s(t)=\exp(t(\sum_{j=1+t_r}^{s+t_r}X^{A_1}_{jk}-\frac{(1+\varepsilon_{x,r})s U_k}{T}))$. It should be noted that most of the above analysis is similar as the derivation of inequality~\eqref{equ:18} except that $Var(X_{jk}^{\widetilde{P}_2}) \le \bar{a}_k\E[X_{ik}^{\widetilde{P}_2}] \le \bar{a}_k\frac{U_k}{T}$. 

Next for the lower bound, we set $Z^{\widetilde{P}_{2}}_{jk}=\bar{a}_{k}-X^{\widetilde{P}_{2}}_{jk}$ and $Z^{A_1}_{jk}=\bar{a}_{k}-X^{A_1}_{jk}$, then we have 
\begin{equation}\label{equ:39}
	\begin{aligned}
		&P\left(\sum_{j=1+t_r}^{s+t_r}Z_{jk}^{A_1}+\sum_{j=s+t_r}^{t_{r+1}}Z^{\widetilde{P}_{2}}_{jk}\ge\frac{(1+\varepsilon_{x,r})t_r\Big((1-\varepsilon)T\bar{a}_{k}-L_{k}\Big)}{T}\right)\\
		&\le \min_{t>0}\E\left[\exp\left(t\bigg(\sum_{j=1+t_r}^{s+t_r}Z^{A_1}_{jk}+\sum_{j=s+t_r}^{t_{r+1}}Z^{\widetilde{P}_{2}}_{jk}-\frac{(1+\varepsilon_{x,r})t_r \Big((1-\varepsilon)T\bar{a}_{k}-L_{k}\Big)}{T}\bigg)\right)\right]\\
    \end{aligned}
\end{equation}
\begin{align*}
		&\le \min_{t>0}\E\left[\exp\left(t\bigg(\sum_{j=1+t_r}^{s+t_r}Z^{A_1}_{jk}-\frac{(1+\varepsilon_{x,r})s \Big((1-\varepsilon)T\bar{a}_{k}-L_{k}\Big)}{T}\bigg)\right.\right.\\
		&\qquad\qquad\qquad\left.\left.+t\bigg(\sum_{j=s+t_r}^{t_{r+1}}Z^{\widetilde{P}_{2}}_{jk}-\frac{(1+\varepsilon_{x,r})(t_r-s) \Big((1-\varepsilon)T\bar{a}_{k}-L_{k}\Big)}{T}\bigg)\right)\right]\\
		&\le \min_{t>0}\mathbb{E}\left[\varphi_k^s(t)\exp\left(t\Big(\sum_{j=s+t_r}^{t_{r+1}}(Z^{\widetilde{P}_{2}}_{jk}-\mathbb{E}[Z^{\widetilde{P}_{2}}_{jk}])\Big)+\frac{t_r-s}{T}t\Big(T\mathbb{E}[Z^{\widetilde{P}_{2}}_{jk}]-(1+\varepsilon_{x,r})((1-\varepsilon)T\bar{a}_{k}-L_{k})\Big)\right)\right] \\
		&\le \min_{t>0}\mathbb{E}\left[\varphi_k^s(t)\exp\left((t_r-s)\frac{\sigma^2}{\bar{a}_k^2}(e^{t\bar{a}_k}-1-t\bar{a}_k)+\frac{-(t_r-s)t\varepsilon_{x,r}\Big((1-\varepsilon)T\bar{a}_{k}-L_{k}\Big)}{T}\right)\right]\\
		&\le \mathbb{E}\left[\varphi_k^s\left(\frac{\ln(1+\varepsilon_{x,r})}{\bar{a}_k}\right)\exp\left(-\frac{(t_r-s)(1-\varepsilon)(T\bar{a}_{k}-L_{k})}{T\bar{a}_k}\Big((1+\eta)\ln(1+\eta)-\eta\Big)\right)\right]\\
		&\le \mathbb{E}\left[\varphi_k^s\left(\frac{\ln(1+\varepsilon_{x,r})}{\bar{a}_k}\right)\exp\left(-\frac{t_r-s}{T}\frac{\varepsilon_{x,r}^2}{4\gamma_1}\right) \right]
\end{align*}
where the third inequality follows from setting $\varphi_k^s(t)=\exp\left(t(\sum_{j=1+t_r}^{s+t_r}Z^{A_1}_{jk}-\frac{(1+\varepsilon_{x,r})s ((1-\varepsilon)T\bar{a}_{k}-L_{k})}{T})\right)$; the fifth inequality from $T\mathbb{E}[Z^{\widetilde{P}_{2}}_{jk}]\le(1-\varepsilon)T\bar{a}_{k}-L_{k}$ and $\sigma^{2}=Var(Z^{\widetilde{P}_{2}}_{jk})\le\bar{a}_{k}\mathbb{E}[Z^{\widetilde{P}_{2}}_{jk}]$.

Then we consider the revenue and have that 
\begin{equation}\label{equ:40}
	\begin{aligned}
		&P(\sum_{j=1+t_r}^{s+t_r}Y_{j}^{A_1}+\sum_{j=s+1+t_r}^{t_{r+1}}Y^{\widetilde{P}_{2}}_{j}\le (1-\varepsilon_{y,r})\frac{t_{r}}{T}Z^r)\\
		&\le \min_{t>0}\E\left[\exp\left(t\Big((1-\varepsilon_{y,r})\frac{t_{r}}{T}Z^r-\sum_{j=1+t_r}^{s+t_r}Y^{A_1}_{j}-\sum_{j=s+1+t_r}^{t_{r+1}}Y^{\widetilde{P}_{2}}_{j})\Big)\right)\right]\\
		&\le \min_{t>0}\mathbb{\E}\left[\exp\left(t\Big(\frac{s}{T}(1-\varepsilon_{y,r})Z^r-\sum_{j=1+t_r}^{s+t_r}Y^{A_1}_{j}\Big)+t\Big(\frac{t_r-s}{T}(1-\varepsilon_{y,r})Z^r-\sum_{j=s+1+t_r}^{t_{r+1}}Y^{\widetilde{P}_{2}}_{j}\Big)\right)\right]\\
		&\le \min_{t>0}\mathbb{E}\left[\psi^s(t)\exp\left(t\sum_{j=s+1+t_r}^{t_{r+1}}(\mathbb{E}[Y^{\widetilde{P}_{2}}_{j}]-Y^{\widetilde{P}_{2}}_{j})+\frac{t_r-s}{T}t\Big((1-\varepsilon_{y,r})Z^r-\mathbb{E}[Y^{\widetilde{P}_{2}}_{j}]\Big)\right)\right] \\
		&\le \min_{t>0}\mathbb{E}\left[\psi^s(t)\exp\left((t_r-s)\frac{\sigma_1^2}{\bar{w}^2}(e^{t\bar{w}}-1-t\bar{w})-(t_r-s)t\varepsilon_{y,r} \mathbb{E}[Y_{j}^{\widetilde{P}_{2}}]\right)\right]\\
		&\le \mathbb{E}\left[\psi^s\left(\frac{\ln(1+\varepsilon_{y,r})}{\bar{w}}\right)\exp\left(-\frac{(t_r-s)\mathbb{E}[Y_j^{\widetilde{P}_{2}}]}{\bar{w}}\Big((1+\eta)\ln(1+\eta)-\eta\Big)\right)\right]\\
		&\le \mathbb{E}\left[\psi^s\left(\frac{\ln(1+\varepsilon_{y,r})}{\bar{w}}\right)\exp\left(-\frac{t_r-s}{T}\frac{\varepsilon^2_{y,r}Z^r}{4\bar{w}}\right) \right]
	\end{aligned}
\end{equation} where the third inequality follows from $\psi^s(t)=\exp(t(\frac{s}{T}(1-\varepsilon_{y,r})Z^r-\sum_{j=1+t_r}^{s+t_r}Y^{A_1}_{j}))$; the fourth inequality from $Z^r\le T\mathbb{E}[Y^{\widetilde{P}_{2}}_{j}]$; the fifth inequality from $\sigma_1^2=Var(Y^{\widetilde{P}_{2}}_{j})\le \bar{w} \E[Y^{\widetilde{P}_{2}}_{j}]$.

With the inequalities~\eqref{equ:38}-\eqref{equ:40}, we can bound the failure probability of hybrid Algorithm $A_1^s\widetilde{P}_2^{t_r-s}$ in stage $r$ by $\mathcal{F}_{r}(A_1^s\widetilde{P}_{2}^{t_{r}-s})$ which is defined as
\begin{equation}
	\begin{aligned}
		\mathcal{F}_{r}(A_1^s\widetilde{P}_{2}^{t_{r}-s})=&\E\bigg[\phi_k^s\left(\frac{\ln(1+\varepsilon_{x,r})}{\bar{a}_k}\right)\exp\left(-\frac{t_r-s}{T}\frac{\varepsilon_{x,r}^2}{4\gamma_1}\right) + \varphi_k^s\left(\frac{\ln(1+\varepsilon_{x,r})}{\bar{a}_k}\right)\exp\left(-\frac{t_r-s}{T}\frac{\varepsilon_{x,r}^2}{4\gamma_1}\right) \\&+ \psi^s\left(\frac{\ln(1+\varepsilon_{y,r})}{\bar{w}}\right)\exp\left(-\frac{t_r-s}{T}\frac{\varepsilon^2_{y,r}Z^r}{4\bar{w}}\right)\bigg]
	\end{aligned} 
\end{equation}

\begin{lemma}\label{lemma:6}
$\mathcal{F}_r(A_1^s\widetilde{P}_2^{t_r-s})\le \mathcal{F}_r(A_1^{s-1}\widetilde{P}_2^{t_r-s + 1})$
\end{lemma} 
\begin{proof}
By the definition of $\mathcal{F}_r(A^s\widetilde{P}_2^{T-s})$, we have that
\begin{equation*}
	\begin{aligned}
		\mathcal{F}_{r}(A_1^s\widetilde{P}_{2}^{t_{r}-s})=&\E\bigg[\phi_k^s\left(\frac{\ln(1+\varepsilon_{x,r})}{\bar{a}_k}\right)\exp\left(-\frac{t_r-s}{T}\frac{\varepsilon_{x,r}^2}{4\gamma_1}\right) + \varphi_k^s\left(\frac{\ln(1+\varepsilon_{x,r})}{\bar{a}_k}\right)\exp\left(-\frac{t_r-s}{T}\frac{\varepsilon_{x,r}^2}{4\gamma_1}\right) \\&+ \psi^s\left(\frac{\ln(1+\varepsilon_{y,r})}{\bar{w}}\right)\exp\left(-\frac{t_r-s}{T}\frac{\varepsilon^2_{y,r}Z^r}{4\bar{w}}\right)\bigg]\\
		=&\E\bigg[\phi_k^{s-1}\left(\frac{\ln(1+\varepsilon_{x,r})}{\bar{a}_k}\right)\exp\left(\frac{\ln(1+\varepsilon_{x,r})}{\bar{a}_k}\Big(X^{A_1}_{sk}-\frac{(1+\varepsilon_{x,r})U_k}{T}\Big)\right)\exp\left(-\frac{t_r-s}{T}\frac{\varepsilon_{x,r}^2}{4\gamma_1}\right) \\&+ \varphi_k^{s-1}\left(\frac{\ln(1+\varepsilon_{x,r})}{\bar{a}_k}\right)\exp\left(\frac{\ln(1+\varepsilon_{x,r})}{\bar{a}_k}\Big(Z^{A_1}_{sk}-\frac{(1+\varepsilon_{x,r})((1-\varepsilon)T\bar{a}_{k}-L_{k})}{T}\Big)\right)\exp\left(-\frac{t_r-s}{T}\frac{\varepsilon_{x,r}^2}{4\gamma_1}\right) \\&+ \psi^{s-1}\left(\frac{\ln(1+\varepsilon_{y,r})}{\bar{w}}\right)\exp\left(\frac{\ln(1+\varepsilon_{y,r})}{\bar{w}}\Big(\frac{(1-\varepsilon_{y,r})Z^r}{T}-Y^{A_1}_{s}\Big)\right)\exp\left(-\frac{t_r-s}{T}\frac{\varepsilon^2_{y,r}Z^r}{4\bar{w}}\right)\bigg]
	\end{aligned} 
\end{equation*}
According to algorithm A in \cref{alg:2}, we allocate the $s$-th request to the channel $i^{*}$ which minimize the $\mathcal{F}_{r}(A_1^s\widetilde{P}_{2}^{t_{r}-s})$. Thus we have 
\begin{equation*}
	\begin{aligned}
		\mathcal{F}_{r}(A_1^s\widetilde{P}_{2}^{t_{r}-s})\le&\E\bigg[\phi_k^{s-1}\left(\frac{\ln(1+\varepsilon_{x,r})}{\bar{a}_k}\right)\exp\left(\frac{\ln(1+\varepsilon_{x,r})}{\bar{a}_k}\Big(X^{\widetilde{P}_{2}}_{sk}-\frac{(1+\varepsilon_{x,r})U_k}{T}\Big)\right)\exp\left(-\frac{t_r-s}{T}\frac{\varepsilon_{x,r}^2}{4\gamma_1}\right) \\&+ \varphi_k^{s-1}\left(\frac{\ln(1+\varepsilon_{x,r})}{\bar{a}_k}\right)\exp\left(\frac{\ln(1+\varepsilon_{x,r})}{\bar{a}_k}\Big(Z^{\widetilde{P}_{2}}_{sk}-\frac{(1+\varepsilon_{x,r})((1-\varepsilon)T\bar{a}_{k}-L_{k})}{T}\Big)\right)\exp\left(-\frac{t_r-s}{T}\frac{\varepsilon_{x,r}^2}{4\gamma_1}\right) \\&+ \psi^{s-1}\left(\frac{\ln(1+\varepsilon_{y,r})}{\bar{w}}\right)\exp\left(\frac{\ln(1+\varepsilon_{y,r})}{\bar{w}}\Big(\frac{(1-\varepsilon_{y,r})Z^r}{T}-Y^{\widetilde{P}_{2}}_{s}\Big)\right)\exp\left(-\frac{t_r-s}{T}\frac{\varepsilon^2_{y,r}Z^r}{4\bar{w}}\right)\bigg]
	\end{aligned} 
\end{equation*}
Following the similar analysis in the inequality \eqref{equ:38}-\eqref{equ:40}, we can show that
\begin{equation*}
    \begin{aligned}
        \mathcal{F}_{r}(A_1^s\widetilde{P}_{2}^{t_{r}-s}) \le& \E\bigg[\phi_k^{s-1}\left(\frac{\ln(1+\varepsilon_{x,r})}{\bar{a}_k}\right)\exp\left(-\frac{t_r-s+1}{T}\frac{\varepsilon_{x,r}^2}{4\gamma_1}\right) \\&+ \varphi_k^{s-1}\left(\frac{\ln(1+\varepsilon_{x,r})}{\bar{a}_k}\right)\exp\left(-\frac{t_r-s+1}{T}\frac{\varepsilon_{x,r}^2}{4\gamma_1}\right) \\&+ \psi^{s-1}\left(\frac{\ln(1+\varepsilon_{y,r})}{\bar{w}}\right)\exp\left(-\frac{t_r-s+1}{T}\frac{\varepsilon^2_{y,r}Z^r}{4\bar{w}}\right)\bigg]\\
        \le& \mathcal{F}_{r}(A_1^{s-1}\widetilde{P}_{2}^{t_{r}-s+1}),
    \end{aligned}
\end{equation*}
which completes the proof.
\end{proof}

In \cref{lemma:a1}, we have proven that $\mathcal{F}_{r}(\widetilde{P}_{2}^{t_{r}})\le \delta$ and we will show that $\mathcal{F}_r(A_1^s\widetilde{P}_2^{t_r-s})\le \mathcal{F}_r(A_1^{s-1}\widetilde{P}_2^{t_r-s + 1})$ in \cref{lemma:6}. Thus we have $\mathcal{F}_r(A_1^{t_r})\le\mathcal{F}_r(\widetilde{P}_2^{t_r})\le \delta$ by induction. Meanwhile, according to \cref{thm:5} we have that
\begin{align*}
    (1-(4+\frac{2}{\xi^*-\varepsilon})\varepsilon_{x,r-1})W_{\varepsilon}\le Z_r \le W_{\varepsilon}
\end{align*}
with probability $1-2\delta$.
During the stage $r$, the Algorithm $A_{1}$ return a solution satisfying

\begin{align*}
     &\sum_{j=t_r+1}^{t_{r+1}}X_{jk}^{A_1}\le\frac{(1+\varepsilon_{x,r})t_r}{T}U_r, \forall k \in \mathcal{K} \\
    &\sum_{j=t_r+1}^{t_{r+1}}(\bar{a}_{k}-X_{jk}^{A_1})\le\frac{(1+\varepsilon_{x,r})t_r}{T}((1-\varepsilon)T\bar{a}_{k}-L_{k}), \forall k \in \mathcal{K}\\
    &\sum_{j=t_r+1}^{t_{r+1}}Y_{j}^{A_1}\ge (1-\varepsilon_{y,r})\frac{t_r}{T}Z^r \ge (1-\varepsilon_{y,r})\frac{t_r}{T}(1-(4+\frac{2}{\xi^*-\varepsilon})\varepsilon_{x,r-1})W_{\varepsilon}
\end{align*}
with probability at least $1-3\delta$, since
\begin{align*}
    &P\left( \bigg\{\sum_{j=t_r+1}^{t_{r+1}}X_{jk}^{A_1}\in\Big[\frac{t_{r}}{T}\Big((1+\varepsilon_{x,r})L_{k}-(\varepsilon(1+\varepsilon_{x,r})-\varepsilon_{x,r})T\bar{a}_{k}\Big),\frac{(1+\varepsilon_{x,r})t_r}{T}U_r\Big],\forall k\in\mathcal{K}\bigg\}\right. \\
    &\qquad\left.\bigcap\bigg\{\sum_{j=t_r+1}^{t_{r+1}}Y^{A_1}_{j}\geq(1-\varepsilon_{y,r})\frac{t_r}{T}Z^r \bigg\}\bigcap\bigg\{ Z^r \in  \Big[\frac{t_{r}W_{\varepsilon}}{T}\Big(1-\big(4+\frac{1}{\xi^*-\varepsilon}\big)\varepsilon_{x,r}\Big),  \frac{t_{r}W_{\varepsilon}}{T}\Big(1+\big(\frac{1}{\xi^*-\varepsilon}\big)\varepsilon_{x,r}\Big)\Big]\bigg\}\right)\\
    &\geq (1-\delta) (1-2\delta)\geq 1-3\delta.
\end{align*}
Now considering all the stages, for the upper bound, we have
\begin{align}\label{equ:42}
     \sum_{r=0}^{l-1}\sum_{j=t_r+1}^{t_{r+1}}X_{jk}^{A_1}\le\sum_{r=0}^{l-1}\frac{(1+\varepsilon_{x,r})t_r}{T}U_{k}\le U_{k}. 
 \end{align}
    For the lower bound, we have
\begin{equation}\label{equ:43}
\begin{aligned}
    (1-\varepsilon)T\bar{a}_{k}-\sum_{r=0}^{l-1}\sum_{j=t_r+1}^{t_{r+1}}X_{jk}^{A_1}&=\sum_{r=0}^{l-1}\sum_{j=t_r+1}^{t_{r+1}}(\bar{a}_{k}-X_{jk}^{A_1})\\&\le\sum_{r=0}^{l-1}\frac{(1+\varepsilon_{x,r})t_r}{T}((1-\varepsilon)T\bar{a}_{k}-L_{k})\\&\le(1-\varepsilon)T\bar{a}_{k}-L_{k}.
    \end{aligned}
\end{equation}
which is equivalent to
\begin{align*}
    \sum_{r=0}^{l-1}\sum_{j=t_r+1}^{t_{r+1}}X_{jk}^{A_1}\ge L_k.
\end{align*}
And for the revenue, we have
\begin{equation}\label{equ:44}
    \begin{aligned}
    \sum_{r=0}^{l-1}\sum_{j=t_r+1}^{t_{r+1}}Y_{j}^{A_1}&\ge\sum_{r=0}^{l-1}(1-\varepsilon_{y,r})\frac{t_r(1-(4+\frac{2}{\xi^*-\varepsilon})\varepsilon_{x,r-1}}{T})W_{\varepsilon}
    \\&\ge\sum_{r=0}^{l-1}(1-\varepsilon_{y,r})\frac{t_r(1-(4+\frac{2}{\xi^*-\varepsilon})\varepsilon_{x,r-1})}{T}(1-\frac{\varepsilon}{\xi^*})W_0\\&\ge (1-O(\frac{\varepsilon}{\xi^*-\varepsilon}))W_0.
    \end{aligned}
\end{equation}
since $\delta = \frac{\varepsilon}{3l}$, the inequalities~\eqref{equ:42}-\eqref{equ:44} hold with probability at least $1-\varepsilon$, which completes the proof of \cref{thm:4}.

\section{Proof of \cref{thm:6}}\label{appendix:thm6}
\thmknownth*
\begin{proof}
\textbf{RHS:}
This side takes the same techniques as in \cref{lemma:5ub}. First, the dual of LP~\eqref{equ:5} in \cref{subsec:framework} is
\begin{equation}\label{equ:48}
    \begin{aligned}
        \min_{\alpha,\beta,\rho}\ &\sum_{k\in\mathcal{K}} \alpha_{k}U_{k}-\sum_{k\in\mathcal{K}}\beta_{k}L_{k}+\sum_{j\in\mathcal{J}}T p_{j}\rho_j \\
		s.t.\ &\sum_{k\in\mathcal{K}}(\alpha_{k}-\beta_{k})a_{ijk}+\rho_{j}\ge0\ \forall i\in\mathcal{I},j\in\mathcal{J}\\
		&\sum_{k\in\mathcal{K}}T\bar{a}_{k}\beta_{k}=1\\
 		&\alpha_{k},\beta_{k},\rho_{j}\ge 0,k\in\mathcal{K}, j\in\mathcal{J}.
    \end{aligned}
\end{equation} 
We denote the optimal solution of LP~\eqref{equ:5} in \cref{subsec:framework} and LP~\eqref{equ:48} as $(x_{ij}^{*},\xi^{*})$ and $(\alpha_k^{*},\beta_{k}^{*},\rho_{k}^{*})$ respectively.

According to the KKT conditions\citep{boyd2004convex}, we have that
\begin{equation}\label{equ:kkt}
    \begin{aligned}
    &\sum_{k\in\mathcal{K}}(\alpha^{*}_{k}-\beta^{*}_{k})a_{ijk}x_{ij}^{*}+\rho_{j}^{*}x_{ij}^{*}=0\\
    &\sum_{k\in\mathcal{K}}T\bar{a}_{k}\beta^{*}_{k}=1\\
    &\rho_j^{*}(\sum_{i\in\mathcal{I}} x_{ij}^{*}-1)=0 \\
    &\alpha_{k}^{*}(\sum_{ij}T p_{j}a_{ijk}x_{ij}^{*}-U_k)=0\\
    &\beta_{k}^{*}(L_{k}+\xi^{*} T\bar{a}_{k}-\sum_{ij}T p_{j}a_{ijk}x_{ij}^{*})=0
     \end{aligned}
\end{equation}

Similarly, the dual of sampled LP~\eqref{equ:47} in \cref{alg:M} is

\begin{equation}\label{equ:50}
    \begin{aligned}
        \min_{\alpha,\beta,\rho}\ &\sum_{k\in\mathcal{K}} \alpha_{k}\frac{t_r}{T}U_{k}-\sum_{k\in\mathcal{K}}\beta_{k}\frac{t_r}{T}L_{k}+\sum_{j\in\mathcal{S}_r}\rho_{j} \\
		s.t.\ &\sum_{k\in\mathcal{K}}(\alpha_{k}-\beta_{k})a_{ijk}+\rho_{j}\ge0\ \forall i\in\mathcal{I},j\in\mathcal{S}_r\\
		&\sum_{k\in\mathcal{K}}t_r\bar{a}_{k}\beta_{k}=1\\
 		&\alpha_{k},\beta_{k},\rho_{j}\ge 0,k\in\mathcal{K}, j\in\mathcal{S}_r
    \end{aligned}
\end{equation}
where $\mathcal{S}_r$ denotes the request set in stage $r$.
Since $(\alpha_k^{*},\beta_{k}^{*},\rho_{k}^{*})$ is a feasible solution to the LP~\eqref{equ:48}, the solution $(\frac{T\alpha_k^{*}}{t_r},\frac{T\beta_{k}^{*}}{t_r},\frac{T\rho_{k}^{*}}{t_r})$ is feasible for the dual of sample LP~\eqref{equ:50}, we have that \begin{equation}\label{equ:51}
    \begin{aligned}
    \widehat{\xi}+2\epsilon_{x,r}=&\frac{T}{t_r}\left(\sum_{k\in\mathcal{K}} \alpha^{*}_{k}\frac{t_r}{T}U_{k}-\sum_{k\in\mathcal{K}}\beta^{*}_{k}\frac{t_r}{T}L_{k}+\sum_{j\in\mathcal{S}_r}\rho^{*}_{j}\right)\\
    &=\underbrace{\frac{T}{t_r}\sum_{k\in\mathcal{K}} \alpha_{k}^{*}\left(\frac{t_r}{T}U_{k}-\sum_{j\in \mathcal{S}_r,i\in\mathcal{I}}a_{ijk}x_{ij}^{*}\right)}_{\textcircled{1}}+\underbrace{\frac{T}{t_r}\sum_{k\in\mathcal{K}}\beta_{k}^{*}\left(\sum_{j\in \mathcal{S}_r,i\in\mathcal{I}}a_{ijk}x_{ij}^{*}-\frac{t_r}{T }L_{k}\right)}_{\textcircled{2}}\\&+\underbrace{\frac{T}{t_r}\sum_{j\in \mathcal{S}_r}\left(\rho^{*}_{j}+\sum_{i\in\mathcal{I},k\in\mathcal{K}}(\alpha^{*}_{k}-\beta^{*}_{k})a_{ijk}x_{ij}^{*}\right)}_{\textcircled{3}}\\
\end{aligned}
\end{equation}
\begin{align*}
    =\underbrace{\frac{T}{t_r}\sum_{k\in\mathcal{K}} \alpha_{k}^{*}\left(\frac{t_r}{T}U_{k}-\sum_{j\in \mathcal{S}_r,i\in\mathcal{I}}a_{ijk}x_{ij}^{*}\right)}_{\textcircled{1}}+\underbrace{\frac{T}{t_r}\sum_{k\in\mathcal{K}}\beta_{k}^{*}\left(\sum_{j\in \mathcal{S}_r,i\in\mathcal{I}}a_{ijk}x_{ij}^{*}-\frac{t_r}{T }L_{k}\right)}_{\textcircled{2}}
\end{align*} 
where the final equality follows from the KKT conditions \eqref{equ:kkt}, i.e. $\sum_{k\in\mathcal{K}}(\alpha^{*}_{k}-\beta^{*}_{k})a_{ijk}x_{ij}^{*}+\rho_{j}^{*}x_{ij}^{*}=0$ and $\rho_j^{*}(\sum_{i\in\mathcal{I}} x_{ij}^{*}-1)=0$, so that $\sum_{i\in\mathcal{I}, k\in\mathcal{K}}(\alpha^{*}_{k}-\beta^{*}_{k})a_{ijk}x_{ij}^{*}+\sum_{i\in\mathcal{I}}\rho_{j}^{*}x_{ij}^{*}=\sum_{i\in\mathcal{I}, k\in\mathcal{K}}(\alpha^{*}_{k}-\beta^{*}_{k})a_{ijk}x_{ij}^{*}+\rho_{j}^{*}=0$.

For those $k$ such that $L_k+\xi^{*} T\bar{a}_{k} <\sum_{i\in\mathcal{I},j\in\mathcal{J}}T p_{j}a_{ijk}x_{ij}^{*}<U_k $, we know that they have no effect to $\widehat{\xi}$ following the complementary slackness in \eqref{equ:kkt}.
For part $\textcircled{1}$, we only consider the resource $k$ making $\sum_{i\in\mathcal{I},j\in\mathcal{J}}T p_{j}a_{ijk}x_{ij}^{*}=U_k$. By \cref{lemma:2}
, it is easy to get that
\begin{equation}\label{equ:52}
    \begin{aligned}
        P(\sum_{j\in[t_{r}],i\in\mathcal{I}}a_{ijk}x_{ij}^{*}
        \le(1-\epsilon_{x,r})\frac{t_r}{T}U_{k})
        \le \exp\bigg(-\frac{\frac{t_r}{T}\epsilon_{x,r}^2}{2(1+\frac{\epsilon_{x,r}}{3})\frac{\bar{a}_{k}}{U_k}}\bigg)
    \end{aligned}
\end{equation} where $\mathbb{E}(\sum_{i\in\mathcal{I}}a_{ijk}x_{ij}^{*})=\frac{U_{k}}{T}$, $\forall j\in \mathcal{J}$.

Similarly, for part $\textcircled{2}$, we only consider the constraints $k$ making $\sum_{i\in\mathcal{I},j\in\mathcal{J}}T p_{j}a_{ijk}x_{ij}^{*}=L_k+\xi^{*} T\bar{a}_{k}$. Before that, we redefine the r.v. $Y_{jk}=(1+\xi^{*})\bar{a}_{k}-\sum_{i\in\mathcal{I}}a_{ijk}x_{ij}^{*}$. Since $\xi^{*}\in[0,1]$ from Assumption~3, we know that $|Y_{jk}|\le(1+\xi^{*})\bar{a}_{k}$ and $E(Y_{jk})=\frac{T\bar{a}_{k}-L_{k}}{T}$. Therefore, by \cref{lemma:2},
\begin{equation}\label{equ:53}
    \begin{aligned}
        P(\sum_{j\in\mathcal{S}_r}Y_{jk}\le(1-\epsilon_{x,r})\frac{t_{r}}{T}(T\bar{a}_{k}-L_{k}))\le \exp\left(-\frac{\frac{t_r}{T}\epsilon_{x,r}^2}{2(1+\frac{\epsilon_{x,r}}{3})\frac{(1+\xi^{*})\bar{a}_{k}}{T\bar{a}_{k}-L_{k}}}\right)
    \end{aligned}
\end{equation}
Since $\gamma_{2}=\max(\frac{\bar{a}_{k}}{U_{k}}, \frac{\bar{a}_{k}}{T\bar{a}_{k}-L_{k}})=O(\frac{\epsilon^{2}}{\ln (\frac{K}{\varepsilon})})$, and both lower and upper bound are achieved only if $U_k = T\bar{a}_k-L_k$, we have that
\begin{equation}\label{equ:54}
    \begin{aligned}
    &\sum_{j\in\mathcal{S}_r,i\in\mathcal{I}}a_{ijk}x_{ij}^{*}\ge(1-\epsilon_{x,r})\frac{t_r}{T}U_{k}\\
    &\sum_{j\in\mathcal{S}_r}((1+\xi^{*})\bar{a}_{k}-\sum_{i\in\mathcal{I}}a_{ijk}x_{ij}^{*})\ge(1-\epsilon_{x,r})\frac{t_r}{T}(T\bar{a}_{k}-L_{k})\\
    \end{aligned}
\end{equation} w.p. at least $1-\delta$. 

Therefore, with probability at least $1-\delta$ we have that
\begin{equation}\label{equ:55}
    \begin{aligned}
    &\textcircled{1}+\textcircled{2}\\
    &=\frac{T}{t_r}\left(\sum_{k\in\mathcal{K}} \alpha_{k}^{*}\bigg(\frac{t_r}{T}U_{k}-\sum_{j\in \mathcal{S}_r,i\in\mathcal{I}}a_{ijk}x_{ij}^{*}\bigg)+\sum_{k\in\mathcal{K}}\beta_{k}^{*}\bigg(\sum_{j\in \mathcal{S}_r,i\in\mathcal{I}}a_{ijk}x_{ij}^{*}-\frac{t_r}{T }L_{k}\bigg)\right)\\
  &=\frac{T}{t_r}\left(\sum_{k\in\mathcal{K}} \alpha_{k}^{*}\bigg(\frac{t_r}{T}U_{k}-\sum_{j\in \mathcal{S}_r,i\in\mathcal{I}}a_{ijk}x_{ij}^{*}\bigg)+\sum_{k\in\mathcal{K}}\beta_{k}^{*}\bigg(\frac{t_r}{T }(T\bar{a}_{k}-L_{k})-\sum_{j\in \mathcal{S}_r}((1+\xi^{*})\bar{a}_{k}-\sum_{i\in\mathcal{I}}a_{ijk}x_{ij}^{*})\bigg)\right.\\&\left.+\sum_{k\in\mathcal{K}}\beta_{k}^{*}\xi^{*}t_r \bar{a}_{k}\right)\\
  &\le\frac{T}{t_r}\bigg(\epsilon_{x,r}\sum_{k\in\mathcal{K}} \alpha_{k}^{*}\frac{t_r}{T}U_{k}+\epsilon_{x,r}\sum_{k\in\mathcal{K}}\beta_{k}^{*}\frac{t_r}{T }(T\bar{a}_{k}-L_{k})+\sum_{k\in\mathcal{K}}\beta_{k}^{*}\xi^{*}t_r\bar{a}_{k}\bigg)\\
  &=\frac{T}{t_r}\bigg(\epsilon_{x,r}(\sum_{k\in\mathcal{K}} \alpha_{k}^{*}\frac{t_r}{T}U_{k}-\sum_{k\in\mathcal{K}}\beta_{k}^{*}\frac{t_r}{T }L_{k})+(\epsilon_{x,r}+\xi^{*})\sum_{k\in\mathcal{K}}\beta_{k}^{*}t_r\bar{a}_{k}\bigg)\\
  &\le\frac{T}{t_r}\bigg(\epsilon_{x,r}\xi^{*}\frac{t_r}{T}+(\epsilon_{x,r}+\xi^{*})\frac{t_r}{T}\bigg)\\
  &=\xi^{*}+(\xi^{*}+1)\epsilon_{x,r}\\
  &\le\xi^{*}+2\epsilon_{x,r}
\end{aligned}
\end{equation} where the first inequality from $\sum_{j\in\mathcal{S}_r,i\in\mathcal{I}}a_{ijk}x_{ij}^{*}\ge(1-\epsilon_{x,r})\frac{t_r}{T}U_{k}$ and $\sum_{j\in\mathcal{S}_r}((1+\xi^{*})\bar{a}_{k}-\sum_{i\in\mathcal{I}}a_{ijk}x_{ij}^{*})\ge(1-\epsilon_{x,r})\frac{t_r}{T}(T\bar{a}_{k}-L_{k})$; the second inequality from $\sum_{k\in\mathcal{K}}\beta^{*}_{k}T\bar{a}_{k}=1$ and $\xi^{*}=\sum_{k\in\mathcal{K}} \alpha^{*}_{k}U_{k}-\sum_{k\in\mathcal{K}}\beta^{*}_{k}L_{k}+\sum_{j\in\mathcal{J}}T p_{j}\rho^{*}_j\ge\alpha^{*}_{k}U_{k}-\sum_{k\in\mathcal{K}}\beta^{*}_{k}L_{k}$; the last inequality follows from the fact $\xi^*\leq 1$.

\paragraph{LHS:}
We design an algorithm~$\widetilde{P}_{3}$ by allocating request $j$ to channel $i$ with probability $(1-\epsilon_{x,r})x_{ij}^{*}$, where $(x_{ij}^{*},\xi^{*})$ is the optimal solution for LP~\eqref{equ:47}. Following the very similar proofs in \cref{sec:known distribution} and letting $\gamma_2 =\max(\frac{\bar{a}_{k}}{U_{k}}, \frac{\bar{a}_{k}}{T\bar{a}_{k}-L_{k}})$, $\varepsilon_{x,r} = \sqrt{\frac{4\gamma_2T\ln(K/\delta)}{t_r}}$, we have that
\begin{align*}
    &P\bigg(\sum_{j=1}^{t_r}X^{\widetilde{P}_{3}}_{jk}\ge \frac{t_r}{T}U_k\bigg)\le \exp\bigg(-\frac{t_r\epsilon_{x,r}^2/T}{2(1-\frac{2}{3}\epsilon_{x,r})\frac{\bar{a}_{k}}{U_k}}\bigg) \le\frac{\delta}{2K}
\end{align*}
where the second inequality follows from the definition of $t_r$ and $\gamma_2=O\big(\frac{\epsilon^{2}}{\ln (K/\varepsilon)}\big)$, which result in $\varepsilon_{x,r}<1$.
Defining $Y_{jk}^{\widetilde{P}_{3}}=(1-\epsilon_{x,r})(1+\xi^{*})\bar{a}_{k}-X^{\widetilde{P}_{3}}_{jk}$, we have that $E(Y_{jk}^{\widetilde{P}_{3}})\le \frac{(1-\epsilon_{x,r})(T\bar{a}_{k}-L_{k})}{T}$ and $|Y_{jk}|\le(1-\epsilon_{x,r})(1+\xi^{*})\bar{a}_{k}$, since $|X^{\widetilde{P}_{3}}_{jk}|\le(1-\epsilon_{x,r})\bar{a}_{k}$. Therefore, we have
\begin{align*}
    &P\bigg(\sum_{j=1}^{t_r}Y^{\widetilde{P}_{3}}_{jk}\ge\frac{t_{r}}{T}(T\bar{a}_{k}-L_{k})\bigg)
		\le \exp\bigg(-\frac{t_r\epsilon_{x,r}^2/T}{2(1-\frac{2}{3}\epsilon_{x,r})\frac{(1-\epsilon_{x,r})(1+\xi^{*})\bar{a}_{k}}{T\bar{a}_{k}-L_{k}}}\bigg)
		\le\frac{\delta}{2K}
\end{align*}
Thus, with probability at least $1-\delta$, we could find  a solution whose consumed resource for each $k$ is in $\bigg[\frac{t_r}{T}(L_k+(\xi^{*}-2\epsilon_{x,r})T\bar{a}_{k}),\frac{t_r}{T}U_k\bigg]$ by Algorithm $\widetilde{P}_{3}$. According to the definition of $\widehat{\xi}$ in \cref{alg:M}, we have that $\widehat{\xi}\ge \xi^{*}-4\epsilon_{x,r}$.

In conclusion, we have that $\xi^{*}-4\epsilon_{x,r}\le\widehat{\xi}\le\xi^{*}$, w.p. $1-2\delta$.

\end{proof}

Now $\widehat{\xi}_{r}$ can be viewed as an good estimate for $\xi^{*}$ from 0, if we have enough data.

\section{Proof of \cref{thm:7}}\label{appendix:g}

\begin{proof}
 We mainly consider two events, namely, 
\begin{align*}
    &G_{1}=\left\{\sum_{j=1}^{T}Y^{A_{2}}_{j}\ge(1-O(\frac{\varepsilon}{\xi^{*}-4\sqrt{\varepsilon}-\varepsilon}))W_0, \sum_{j=1}^{T}X_{jk}^{A_{2}}\in[L_{k},U_{k}] \forall k\in\mathcal{K}\right\},\\
    &G_{2}=\left\{\xi^{*}-4\sqrt{\varepsilon}\le\xih_{0}\le\xi^{*}\right\}.
\end{align*}

\textbf{Step 1}.

When initializing Algorithm~6, we use the first $\varepsilon T$ incoming requests to estimate the optimal measure of feasibility $\xi^{*}$. From the \cref{thm:6}, we have that $\mathrm{P}\big(\{ \xi^{*}-4\sqrt{\varepsilon}\le\xih_{0}\le\xi^{*}\}\big)\ge1-2\delta$, choosing $\delta=\frac{\varepsilon}{3l+2}$.

\textbf{Step 2}.

We investigate the conditional event $G_{1}|G_{2}$. Under the assumption  $\frac{\sqrt{\varepsilon}}{1-\sqrt{\varepsilon}}+4\sqrt{\varepsilon}+\varepsilon\le\xi^{*}$, if $\xi^{*}-4\sqrt{\varepsilon}\le\xih_{0}$, we have $\xih_{0}\ge\xi^{*}-4\sqrt{\varepsilon}\ge\frac{\sqrt{\varepsilon}}{1-\sqrt{\varepsilon}}+\varepsilon$. Besides, we know that the domain 
\begin{align*}
     &L_{k}+\xih_{0} T\bar{a}_{k}\le\sum_{i\in\mathcal{I},j\in\mathcal{J}} Tp_{j}a_{ijk}x_{ij}\le U_{k}, \forall k\in\mathcal{K}\\
    &\sum_{i\in\mathcal{I}} x_{ij}\le 1, \forall j\in\mathcal{J}\\
    &x_{ij}\ge 0, \forall i\in\mathcal{I}, j\in\mathcal{J},
\end{align*}
is feasible under the assumption.

When $\gamma_{3}=O\big(\frac{\varepsilon^{2}}{\ln (K/\varepsilon)}\big)$, according to the \cref{thm:4}, we have
    \begin{align*}
        \mathrm{P}\left(\bigg\{\sum_{j=1}^{T}Y^{A_{2}}_{j}\ge(1-O(\frac{\varepsilon}{\xih_{0}-\varepsilon}))W_0\bigg\}\bigcap\bigg\{ \sum_{j=1}^{T}X_{jk}^{A_{2}}\in[L_{k},U_{k}]\bigg\}\right)\ge 1-3l\delta,
    \end{align*} where $l=\log_{2}\big(\frac{1}{\varepsilon}\big)$.
Due to $\xih_{0}\ge\xi^{*}-4\sqrt{\varepsilon}$, we also could derive that $\mathrm{P}(G_{1}|G_{2})\ge 1-3l\delta$.

\textbf{Step 3}.

Now we can verify that
\begin{align*}
    \mathrm{P}(G_{1}^{c})&=\mathrm{P}(G_{1}^{c}|G_{2})\mathrm{P}(G_{2})+\mathrm{P}(G_{1}^{c}|G_{2}^{c})\mathrm{P}(G_{2}^{c})\\
    &\le\mathrm{P}(G_{1}^{c}|G_{2})+\mathrm{P}(G_{2}^{c})\\
    &\le 3l\delta+2\delta=\epsilon.
\end{align*}

Therefore, $\mathrm{P}(G_{1})\ge 1-\varepsilon$, if $\varepsilon\geq 0$ and $\tau_1=\frac{\sqrt{\varepsilon}}{1-\sqrt{\varepsilon}}$ such that $\tau_{1}+4\sqrt{\varepsilon}+\varepsilon\le\xi^{*}$ and $\gamma_{3}=O(\frac{\varepsilon^{2}}{\ln (\frac{K}{\varepsilon})})\le \max(\frac{\bar{a}_{k}}{U_{k}}, \frac{\bar{a}_{k}}{(1-\epsilon)T\bar{a}_{k}-L_{k}}, \frac{\bar{w}}{W_{\epsilon+\tau_{1}}})$.

\end{proof}

\end{document}